\documentclass[conference]{IEEEtran}
\usepackage{textcomp}
\usepackage[utf8]{inputenc}
\usepackage{float}
\usepackage{units}
\usepackage{bm}
\usepackage{amsmath}
\usepackage{amsthm}
\usepackage{amssymb}
\usepackage{graphicx}
\PassOptionsToPackage{normalem}{ulem}
\usepackage{ulem}

\makeatletter

\floatstyle{ruled}
\newfloat{algorithm}{tbp}{loa}
\providecommand{\algorithmname}{Algorithm}
\floatname{algorithm}{\protect\algorithmname}

\theoremstyle{plain}
\newtheorem{thm}{\protect\theoremname}
\theoremstyle{definition}
\newtheorem{defn}[thm]{\protect\definitionname}
\theoremstyle{remark}
\newtheorem{rem}[thm]{\protect\remarkname}
\theoremstyle{plain}
\newtheorem{prop}[thm]{\protect\propositionname}

\IEEEoverridecommandlockouts
\usepackage{cite}
\usepackage{amsfonts}
\usepackage{amsmath}
\usepackage{algorithmic}
\usepackage{textcomp}
\usepackage{xcolor}
\usepackage{graphicx}
\usepackage{epstopdf}
\usepackage{etoolbox}
\makeatletter
\patchcmd{\@makecaption}
  {\scshape}
  {}
  {}
  {}
\makeatletter
\patchcmd{\@makecaption}
  {\\}
  {.\ }
  {}
  {}
\makeatother
\usepackage{balance}
\usepackage{fancyhdr}
\usepackage{bm}
\usepackage{booktabs}
\usepackage{multirow}

\makeatother

\providecommand{\definitionname}{Definition}
\providecommand{\propositionname}{Proposition}
\providecommand{\remarkname}{Remark}
\providecommand{\theoremname}{Theorem}

\begin{document}
\title{Neural Autoregressive Flows\\for Markov Boundary Learning}
\author{\IEEEauthorblockN{Khoa Nguyen\textsuperscript{1}, Bao Duong\textsuperscript{1}, Viet
Huynh\textsuperscript{2}, Thin Nguyen\textsuperscript{1}} \IEEEauthorblockA{\textsuperscript{1}Deakin Applied Artificial Intelligence Initiative,
Geelong, Australia} \IEEEauthorblockA{\textsuperscript{2}Edith Cowan University, Perth, Australia}\IEEEauthorblockA{\{khoa.nguyen, b.duong, thin.nguyen\}@deakin.edu.au, v.huynh@ecu.edu.au}}
\maketitle
\begin{abstract}
Recovering Markov boundary---the minimal set of variables that maximizes
predictive performance for a response variable---is crucial in many
applications. While recent advances improve upon traditional constraint-based
techniques by scoring local causal structures, they still rely on
nonparametric estimators and heuristic searches, lacking theoretical
guarantees for reliability. This paper investigates a framework for
efficient Markov boundary discovery by integrating conditional entropy
from information theory as a scoring criterion. We design a novel
masked autoregressive network to capture complex dependencies. A parallelizable
greedy search strategy in polynomial time is proposed, supported by
analytical evidence. We also discuss how initializing a graph with
learned Markov boundaries accelerates the convergence of causal discovery.
Comprehensive evaluations on real-world and synthetic datasets demonstrate
the scalability and superior performance of our method in both Markov
boundary discovery and causal discovery tasks.
\end{abstract}

\begin{IEEEkeywords}
Markov boundary, Markov blanket, causal discovery, information theoretic
learning, autoregressive modeling, causal feature selection.
\end{IEEEkeywords}

\section{Introduction}

The Markov Blanket, introduced by Pearl \cite{pearl2000causality}
in causal Bayesian networks, is a set $\mathbf{S}$ that captures
the local causal mechanisms for a variable $T$ and statistically
shields $T$ from the rest of the system $\mathbf{V}\backslash\mathbf{S}$:
$T \perp\!\!\!\perp \mathbf{V}\backslash\mathbf{S}\mid\mathbf{S}$
\cite{pearl2000causality,koller2009probabilistic}. The Markov boundary
(MB) is defined as the minimal Markov blanket.\footnote{Some literature refers to this minimal set as the Markov blanket \cite{koller2009probabilistic,yu2021unified}.
To distinguish it from larger sets, we adopt the Markov boundary definition
from \cite{pearl2000causality}, and use the abbreviation MB for Markov
boundary.} MBs are crucial in a wide range of domains, including bioinformatics
\cite{liu2022inferring}, neuroscience \cite{hipolito2021markov},
and environmental science \cite{raffa2022markov}. In the realm of
data science and causality, they are prevailingly applied in causal
feature selection \cite{yu2021unified}, increasingly in missing data
imputation \cite{liu2022improving}, divide-and-conquer causal discovery
\cite{dong2025dcilp}, and LLM reasoning \cite{liu2024discovery}.
Under standard faithfulness assumptions, the MB of a variable is unique
without the need for additional experimental data \cite{tsamardinos2003algorithms};
thus, it is theoretically feasible to design a learning algorithm
$\mathcal{L}^{*}$ that approximates the optimal solution with sufficient
observational data. However, MB learning faces significant challenges
in large graphs with nonlinear dependencies and unknown associated
data distributions, which complicate accurate and scalable inference
\cite{wu2023practical}.

Traditionally, MB learning in the constraint-based approach indirectly
uses pairwise conditional independence (CI) tests. This direction
\cite{tsamardinos2003towards,margaritis1999bayesian,pellet2008using,wang2020towards,tsamardinos2003algorithms}
has a solid theoretical foundation but relies on the accuracy of these
tests, making it susceptible to restricted assumptions about data
distribution, insufficient training samples, and the scale of the
separating set \cite{wu2023practical}. The sequential decision-making
process can lead to cascading errors and ignoring true positives \cite{wu2023practical}. 

Score-based approaches, which evolved later but received limited attention,
identify local causal structures by optimizing a quantifiable objective
\cite{wu2023practical,gao2017efficient}. However, commonly used scoring
criteria, such as BDeu and BDe, impose strict assumptions on the data,
mainly designed for discrete or linear Gaussian data \cite{wu2023practical}.
Recently, KMB \cite{wu2023practical} was proposed to address the
complexity of data relationships. It, however, employs the conditional
covariance operator (CCO), a kernel-based design known for its inefficiency
in terms of sample size and dimensionality \cite{duong2023diffeomorphic}.
The learning mechanism of KMB relies on a heuristic algorithm lacking
theoretical justification and must solve a continuous relaxation-based
objective at each step of evolving its predicted MB. This leads to
inefficiencies in learning MBs for all variables in the network.

Regarding the applicability of MBs in follow-up tasks, we focus on
their integration as a warm-start step in causal discovery (CD) \cite{aliferis2010local,pellet2008using,tsamardinos2003algorithms}.
The key advantage lies in guiding CD within a reduced search space
by eliminating obviously irrelevant relationships \cite{aliferis2010local,tsamardinos2003algorithms}.
From an overarching technical viewpoint, this group has evolved alongside
constraint-based MB learning, which approaches CD by first solving
the local causal induction problem and then adopting a local-to-global
learning strategy. As a subset of constraint-based CD, a critical
component of these techniques is the pairwise CI test, where greater
efficiency reflects a reduction in the number of required trials \cite{aliferis2010local,mokhtarian2025recursive}.
In the emerging field of differentiable causal discovery (DCD), state-of-the-art
techniques employ MB learning to construct the moral graph, which
serves as a mask for subsequent DAG estimation \cite{nazaret2024stable,amin2024scalable}.
SDCD \cite{nazaret2024stable} has demonstrated improved stability
and scalability, particularly for sparse graphs common in real-world
scenarios. However, the moral graphs learned during preselection are
not true MB-derived moral graphs but rather sparser approximations
that eliminate some false positives while avoiding a complete graph,
due to the inherent limitations of DCD \cite{vowels2022d}. 

We address the gaps across these worlds: our method advances a novel
class of autoregressive flow-based generative models, aligning them
with conditional entropy from classic information theory to derive
a theoretically justified scoring objective. We introduce \textit{Any-subset
Masked Autoregressive Flows}, also known as \textbf{\uline{F}}lows
for \textbf{\uline{AN}}y-\textbf{\uline{S}}ubset (\textbf{FANS}),
which efficiently model nonlinear relationships and noise distributions
through latent representation learning, thereby transforming intractable
universal entropy into an easily estimable counterpart. A single trained
\textbf{FANS} can also estimate any conditional entropy $H(T|\mathbf{S})$,
where $T$ is an arbitrary random variable and $\mathbf{S}$ is any
subset of variables. This enables its use in our greedy algorithm,
\textbf{\uline{FANS}}-based \textbf{\uline{I}}nformation ga\textbf{\uline{T}}h\textbf{\uline{E}}ring
for \textbf{\uline{MB}} discovery (\textbf{FANSITEMB}), facilitating
seamless integration into CD pipelines, unlike previous score-based
approaches that require separate models for each variable \cite{wu2023practical}. 

Our main contributions are:
\begin{itemize}
\item \textbf{FANS},\footnote{https://github.com/khoangdadk/FANS} a neural
model redesigned from original autoregressive flows \cite{huang2018neural}.
It combines strengths from both masked autoregressive autoencoder
neural networks \cite{germain2015made} and normalizing flows \cite{papamakarios2021normalizing},
enhancing expressiveness and efficacy in approximating underlying
distributions. With amortized training, \textbf{FANS} can \textit{estimate
conditional entropy of arbitrary variables conditioned on any subsets,
thereby avoiding the computational expense of learning a separate
model for each subset}. Its compact design also reduces parameters,
optimizing computational efficiency when dealing with sparse Bayesian
networks.
\item \textbf{FANSITEMB}, a greedy and parallelizable algorithm for discovering
MBs in polynomial time, supported by a thorough theoretical analysis
of error bounds.
\item \textbf{FANSITE-DCD}, a MB-based differentiable causal discovery method
adapted from SDCD, ensuring learning of DAGs from a minimal I-map-based
subspace.
\item Empirical evidence on the effectiveness of the proposed framework
across various synthetic settings and four real-world gene regulatory
networks.
\end{itemize}

\section{Related Works}\label{sec:RW}

A significant portion of MB discovery focuses on constraint-based
algorithms that identify MBs by mining conditional independence (CI)
relations. As a crucial component of such group, the correctness of
every CI test must be assumed, with $\lambda^{2}$-tests and $G^{2}$-tests
commonly employed for discrete cases, and Fisher’s Z-test used for
continuous cases under the assumption of linear relationships and
additive Gaussian noise. Growing-Shrink (GS) \cite{margaritis1999bayesian}
was the first sound MB learning algorithm, subsequently improved by
the IAMB family \cite{tsamardinos2003towards}. A straightforward
approach, Total Conditioning (TC) \cite{pellet2008using}, states
that $X\in\mathcal{MB}_{Y}$ and $Y\in\mathcal{MB}_{X}$ $\Longleftrightarrow$
$X \not\perp\!\!\!\perp Y \mid \mathbf{V}\backslash\{X,Y\}$. Despite
possessing a theoretical foundation thanks to fine-grained pairwise
testing, these methods require a sample size that is exponential to
the MB size and necessitate $\mathcal{O}(d^{2})$ CI tests, where
each test may involve conditioning on a large variable set. EEMB \cite{wang2020towards}
recently proposed simultaneous learning of the parent-child and spouse
sets, achieving superior accuracy compared to other methods in this
group.

Score-based MB learners are central to curiosity-driven research as
they design specific scoring functions to directly identify the MB
of a target, though the literature remains sparse. The key motivation
stems from the uniqueness of MB under the faithfulness assumption,
making it theoretically feasible to solve a formal objective for obtaining
an optimal solution: the smallest subset representing the MB. This
problem, however, is fundamentally a combinatorial optimization task,
which is NP-hard and computationally intractable in polynomial time.
Regarding scoring criteria, SLL \cite{niinimaki2012local} and S\textsuperscript{2}TMB
\cite{gao2017efficient} utilized BDeu, a decomposable score commonly
used in Bayesian structure learning and specifically tailored for
discrete data. Another widely adopted score, the Minimum Message Length
(MML), was integrated in \cite{li2020moral} and demonstrated to be
locally consistent under specific conditions. Recently, the authors
of KMB \cite{wu2023practical} established that the MB corresponds
to the feature subset minimizing the conditional covariance operator
(CCO) $\Sigma_{TT|\mathbf{S}}$ when embedding variables from a Euclidean
space into a reproducing kernel Hilbert space (RKHS). Formally, both
the target space $\mathcal{T}$ and feature space $\mathcal{S}$ are
mapped into RKHS spaces $\mathcal{H_{T}}$ and $\mathcal{H_{\mathcal{S}}}$
using positive definite kernels. For $\forall g\in\mathcal{H_{T}}$,
$\langle g,\Sigma_{TT|\mathcal{MB}}g\rangle_{\mathcal{H_{T}}}=\langle g,\Sigma_{TT|\mathbf{V}}g\rangle_{\mathcal{H_{T}}}$.
Although kernel-based characterization can capture nonlinear relationships,
their performance is constrained by kernel selection and suffers from
scalability issues. Except for S\textsuperscript{2}TMB \cite{gao2017efficient},
none of these methods guarantee the correctness of their follow-up
algorithms in approximating the MB along with their provided scores.
The heuristic-based search algorithm in KMB \cite{wu2023practical}
must solve its objective with every operation on the evolving MB (i.e.,
removing or adding a node), which is computationally inefficient.

MB-based causal discovery (CD) can be categorized into three approaches.
For constraint-based methods, the state-of-the-art MARVEL \cite{mokhtarian2025recursive}
leverages MB information for recursive learning through removable
variables, reducing the number of CI tests to $\mathcal{O}(d\Delta_{in}^{2}2^{\Delta_{in}^{2}})$,
whereas other constraint-based algorithms require at least $\mathcal{O}(d^{2}\Delta_{in}^{2}2^{\Delta_{in}^{2}})$
(where $d$ and $\Delta_{in}$ are the number of variables and the
maximum in-degree of the causal DAG). An emerging group, including
SDCD \cite{nazaret2024stable} and DAT \cite{amin2024scalable}, employs
MB-derived moral graphs to mask their differentiable CD models. They
transform MB learning into a sparse regression problem to enable expressive
neural network integration. As noted, this process is essentially
a graph pruning step based on the SEM assumption. Finally, another
line of methods, DCILP \cite{dong2025dcilp}, utilizes MBs to partition
the graph into subgraphs, enabling parallelized CD.

Our work advances these fields by proposing a new family of masked
autoregressive flows---a subclass of normalizing flows that has gained
attention for its efficacy in density estimation and variational inference
\cite{huang2018neural,de2020block}. Our models can estimate arbitrary
conditional entropy and work in conjunction with an efficient greedy
search algorithm. Normalizing flows have intersected with information
theory in various studies, particularly in mutual information estimation
\cite{duong2023diffeomorphic}, demonstrating superior accuracy, especially
with continuous random variables. We, in turn, also focus on the continuous
instance of the problem, as noted in \cite{duong2023diffeomorphic},
which is more challenging due to the unknown closed-form of data density.
Furthermore, we demonstrate how our robust MB learners enhance DAG
searching from the perspective of CD.

\section{Methodology}\label{sec:Methodology}

\begin{figure}[h]
\includegraphics[width=1\columnwidth]{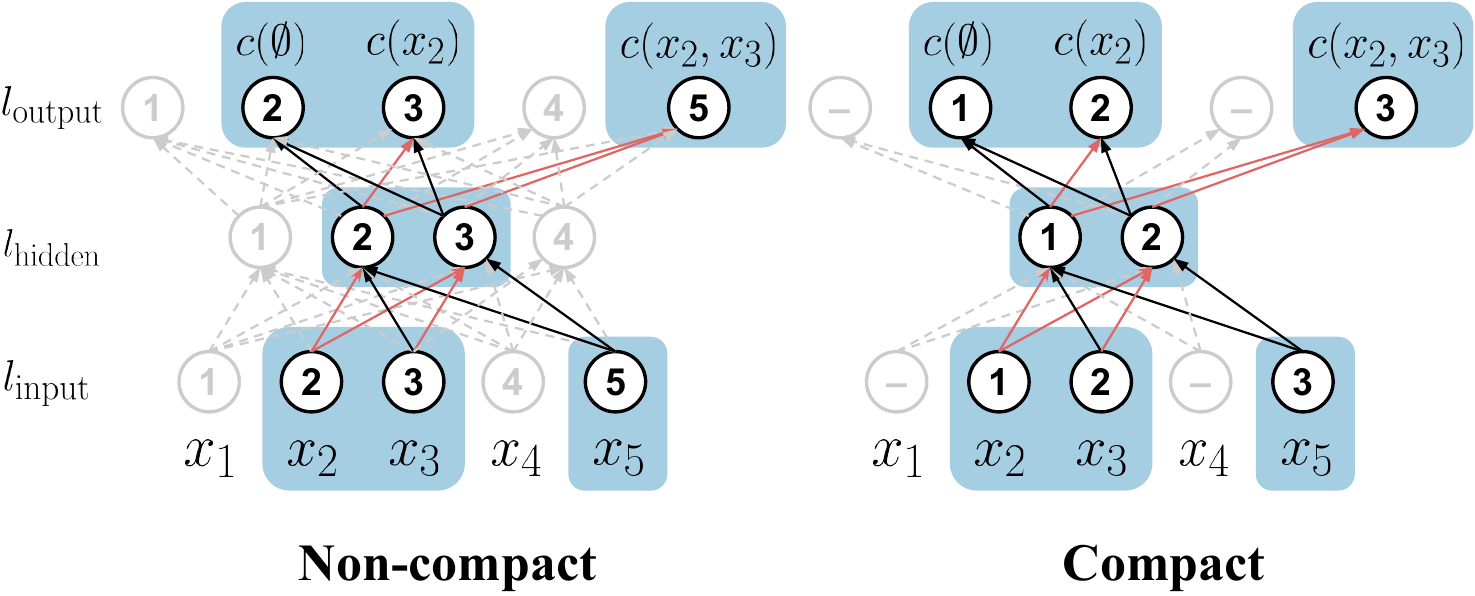}\vspace{0.7cm}

\caption{Illustration of weight masking mechanism in \textbf{FANS}, as described
in Subsection \ref{subsect:asmaf}, for specific subset $\mathbf{S}=\{X_{2},X_{3},X_{5}\}$.
This simplifies the original illustration of MADE architecture for
better visualization, showing only a hidden layer with a block of
nodes and output layer with a block of nodes. Full autoencoder may
consist of multiple hidden layers, each with several blocks of hidden
nodes, which can be permuted, as shown in Fig. 1 from \cite{germain2015made}.
Solid connections represent fully-connected weight regions in Step
1, while red connections indicate autoregressive masks in Step 2.
Other subsets will have different masks. In compact version (right
figure), assuming maximum size of observed subsets is 3, each block
of hidden layer reduces to 2 nodes. Original element indices of $\mathbf{S}$
are rescaled and colliders resolved:\textit{ $[2,3,5]\rightarrow\left[\left\lceil \frac{2*2}{5-1}\right\rceil ,\left\lceil \frac{3*2}{5-1}\right\rceil ,\left\lceil \frac{5*2}{5-1}\right\rceil \right]\rightarrow[1,2,3]$.}}\label{fig:asmaf}
\end{figure}
Regarding notation, we use capitalized letters $X$, $Y$, etc., for
random variables; their bold counterparts $\mathbf{X}$, \textbf{$\mathbf{Y}$},
etc., for sets of variables; lowercase letters $x$, $y$, etc., for
the realizations of those variables; and bold lowercase letters $\mathbf{x}$,
$\mathbf{y}$, etc., for realization vectors of sets. Let us begin
with the formal statement of Markov boundary discovery problem. As
input, we have $N$ $d$-dimensional observations $\mathcal{D}=\left\{ \mathbf{x}^{(k)}\right\} _{k=1}^{N}$,
sampled i.i.d. from an unknown joint distribution \textbf{$\mathbf{P}$}
that is faithful to a corresponding causal Bayesian network (CBN)
$\mathcal{G}$. The goal is to identify the MB for every variable
$X_{i}$.
\begin{defn}[Markov boundary (MB)]
Given a DAG $\mathcal{G}$ and a distribution $\mathbf{P}$ that
are faithful to each other, the MB of each target $X$, denoted as
$\mathcal{MB}_{X}$, is \textit{unique} and consists of its parents
$Pa_{X}$, children $Ch_{X}$, and co-parents (spouses) $Sp_{X}$:
$\mathcal{MB}_{X}=Pa_{X}\cup Ch_{X}\cup Sp_{X}$.
\end{defn}

\subsection{Conditional entropy as a scoring criterion}

We begin by presenting the key observation that underlies the development
of our method.
\begin{thm}
\label{thm:min_hcond}The MB $\mathcal{MB}_{T}\subset\text{\textbf{V}}$
of random variable $T\in\mathbf{V}$ is the subset $\mathbf{S}\subset\mathbf{V}\setminus\{T\}$
that minimizes $H\left(T\mid\mathbf{S}\right)$, where \textbf{$H(\cdot\mid\cdot)$
}is the conditional entropy. Furthermore, $\forall\mathbf{S'}\subseteq\mathbf{V}\setminus\{T\}$
that satisfy $\mathcal{MB}_{T}\subseteq\mathbf{S'}$, we have: $H\left(T\mid\mathbf{S'}\right)=H\left(T\mid\mathcal{MB}_{T}\right)$.
\end{thm}
The proof is provided in the Appendix \ref{subsec:proof_a}.

It is worth emphasizing that estimating conditional entropy is equivalent
to estimating the difference between two marginal entropies. In the
special case of Gaussian random vectors, i.e., assuming that the functional
form of the underlying CBN data-generating process is given by a linear
Gaussian additive noise model $y=f(x)+\epsilon$, where $f(x)=\mathbf{w}^{\mathsf{T}}\mathbf{x}+b$
and $\epsilon\sim\mathcal{N}(\mu,\sigma^{2})$, it is straightforward
to derive $H\left(T|\mathbf{S}\right)$ using the entropy of a multivariate
Gaussian distribution. Given $\mathbf{x}\sim\mathcal{N}(\bm{\mu},\bm{\Sigma})$,
$H(\mathbf{x})=\frac{d}{2}\left(1+\log(2\pi)\right)+\frac{1}{2}\log\det\bm{\Sigma}$.
Let $\bm{\Sigma}_{\mathbf{S}}$ represent the sub-matrix of the covariance
matrix $\bm{\Sigma}$ with row and column set \textit{$\mathbf{S}$}.
From $H\left(T|\mathbf{S}\right)=H(T,\mathbf{S})-H(\mathbf{S})$,
we can derive:
\begin{equation}
\begin{aligned}H\left(T|\mathbf{S}\right) & =\frac{1}{2}\left(1+\log(2\pi)+\log\frac{\det\bm{\Sigma}_{\mathbf{\{}T\mathbf{\}\cup\mathbf{S}}}}{\det\bm{\Sigma}_{\mathbf{S}}}\right)\end{aligned}
\label{eq:gauss_ce_trans}
\end{equation}

Practically, the problem is not simple as real-world CBNs may not
adhere to a deterministic data-generating process, rendering their
closed forms intractable. In these situations, relationships between
continuous random variables are often complex and nonlinear, with
unknown endogenous noise distributions. \textbf{\textit{RQ1.}} How
can we estimate the conditional entropy in such cases? Our key idea
is to leverage the expressiveness and tractability of latent representation
learning through \textbf{\textit{masked autoregressive flows}}.

\subsection{Conditional entropy neural estimator design}

In this subsection, we explain how a masked autoregressive flow (MAF)
map continuous random variables from multimodal distributions into
standard Gaussian space, facilitating conditional entropy estimation
for any underlying distribution. We then introduce our novel class
of MAFs tailored for MB discovery: the \textbf{FANS} model.

\subsubsection{Universal conditional entropy estimation}\label{subsec:ce_diffeo}

At the core of our design are MAFs \cite{papamakarios2017masked,huang2018neural,de2020block},
which have gained significant attention in recent years for critical
applications such as density estimation and variational inference.
MAFs harmonize the strengths of normalizing flows \cite{papamakarios2021normalizing}
and autoregressive models \cite{germain2015made}, offering advantages
from both approaches: the former allows compositions of many invertible
transformations with a tractable Jacobian, and the latter facilitates
the integration of expressive neural architectures while maintaining
efficient GPU training. Consequently, MAFs function as \textit{tractable
universal density approximators} with arbitrary non-zero precision
\cite{huang2018neural,papamakarios2017masked,de2020block}. In principle,
autoregressive modeling first decomposes a joint density $p(\mathbf{x})$
into a product of univariate conditionals: $p(\mathbf{x})=\prod_{i}p(x_{i}\mid\mathbf{x_{<i}})$,
where each univariate conditional $p(x_{i}\mid\mathbf{x_{<i}})$ is
transformed into a base density (e.g., a standard Gaussian) via a
conditional diffeomorphic transformation,\footnote{A diffeomorphic transformation is a map $\tau\left(\cdot\right):X\rightarrow X'$
that is differentiable, invertible, and has a differentiable inverse.} parametrized by $\mathbf{x_{<i}}$, $f_{\mathbf{x_{<i}}}:x_{i}\mapsto u_{i}$:
\begin{equation}
p(x_{i}\mid\mathbf{x_{<i}})=p_{u_{i}}\left(u_{i}\right)\left|\det\frac{\partial f_{\mathbf{x_{<i}}}}{\partial x_{i}}\right|\label{eq:nf_formula}
\end{equation}
where $u_{i}=f_{\mathbf{x_{<i}}}(x_{i})=\tau(x_{i},c(\mathbf{x_{<i}}))$.
Two key features stand out: (1) Such transformation can be deepened
by composing multiple single ones: $f=f_{k}\circ\ldots\circ f_{2}\circ f_{1}$,
and (2) The autoregressive decomposition ensures that each dimension
$u_{i}$ of the transformed joint density also follows the base distribution.
The autoregressive conditioner $c$ is implemented as a masked autoencoder
(MADE) \cite{germain2015made}, which takes a vector \textbf{$\mathbf{x}$}
as input and efficiently computes $c(\mathbf{x_{<i}})$ for all conditionals
within a forward pass. The transformer $\tau$ can be an affine transformation:
$u_{i}=\mu+\sigma x_{i}$ \cite{papamakarios2017masked}, where $\mu$
and $\sigma$ are outputs of $c(\mathbf{x_{<i}})$; or a more expressive
transformation: $u_{i}=\bm{\sigma}^{-1}\left(w^{\mathsf{T}}\cdot\bm{\sigma}(a\cdot x_{i}+b)\right)$
\cite{huang2018neural}, where $\bm{\sigma}$ is a Lipschitz continuous
activation function (e.g., sigmoid), and $w$, $a$, $b$ are output
vectors from $c(\mathbf{x_{<i}})$, with $w$ and $a$ being positive
vectors. Notably, the autoregressive structure ensures that the Jacobian
of transformation $\mathcal{F}$ on multivariate inputs is triangular
by design, making its determinant easily computable:
\begin{equation}
\left|\det\frac{\partial\mathcal{F}}{\partial\mathbf{x}}\right|=\left|\prod_{i}\frac{\partial f_{\mathbf{x_{<i}}}}{\partial x_{i}}\right|\label{eq:ar_decom}
\end{equation}
We will briefly denote $\nicefrac{\partial f_{\mathbf{x_{<i}}}}{\partial x_{i}}$
as $\nicefrac{\partial f}{\partial x_{i}}$ for simplicity. 

\textbf{\textsl{Rendering universal conditional entropy tractable.}}
Given the target variable $T\in\mathbf{V}$, consider a diffeomorphism
$\mathcal{F}_{\mathbf{S}}:\mathbf{x_{S}}\mapsto\mathbf{u}_{\mathbf{S}}$,
where $\mathbf{S}\subseteq\mathbf{V}\setminus\{T\}$. We initially
observe:
\begin{equation}
H(\mathbf{S})=H_{\mathbf{u}}(\mathbf{S})-\mathbb{E}_{p(\mathbf{x}_{\mathbf{S}})}\log\left|\det\frac{\partial\mathcal{F}_{\mathbf{S}}}{\partial\mathbf{x_{S}}}\right|\label{eq:diffeo_ce}
\end{equation}
where $\mathbf{x}_{\mathbf{S}}$ represents the vector part w.r.t
$\mathbf{S}$. Assuming every $\mathcal{F}_{\mathbf{S}}$ maps $\mathbf{x_{S}}$\textbf{
}to $\mathbf{u_{S}}$ in an isometric Gaussian space such that $\mathbf{u_{S}}\sim\mathcal{N}(0,\mathbf{I}_{\mathbf{u_{S}}})$,
we can assert that $H_{\mathbf{u}}(\mathbf{S})=\frac{1}{2}\left(1+\log(2\pi)\right)$.
Let\textbf{ $\mathbf{TS}:=\{T\}\cup\mathbf{S}$}; the conditional
entropy $H\left(T|\mathbf{S}\right)$ for the general case is derived
from Eq. \ref{eq:diffeo_ce} as follows:
\begin{equation}
\begin{aligned}H\left(T|\mathbf{S}\right) & =\mathbb{E}_{p(\mathbf{x_{S}})}\log\left|\det\frac{\partial\mathcal{F}_{\mathbf{S}}}{\partial\mathbf{x_{S}}}\right|-\mathbb{E}_{p(\mathbf{x_{TS}})}\log\left|\det\frac{\partial\mathcal{F}_{\mathbf{TS}}}{\partial\mathbf{x_{TS}}}\right|\\
 & =\mathbb{E}_{p(\mathbf{x_{TS}})}\left[\log\left|\det\frac{\partial\mathcal{F}_{\mathbf{S}}}{\partial\mathbf{x_{S}}}\right|-\log\left|\det\frac{\partial\mathcal{F}_{\mathbf{TS}}}{\partial\mathbf{x_{TS}}}\right|\right]\\
 & =\mathbb{E}_{p(\mathbf{x_{TS}})}\left[\sum_{X_{i}\in\mathbf{S}}\log\left|\frac{\partial f_{\mathbf{S}}}{\partial x_{i}}\right|-\sum_{X_{j}\in\mathbf{TS}}\log\left|\frac{\partial f_{\mathbf{TS}}}{\partial x_{j}}\right|\right]
\end{aligned}
\label{eq:general_ce_trans}
\end{equation}
noting that $f_{\mathbf{S}}(x_{i})=\tau(x_{i},c(\mathbf{x_{<i}}))$
with $X_{i}\in\mathbf{S}$ and $\mathbf{X}_{<i}\subset\mathbf{S}$
including only of the variables prior to $X_{i}$ when decomposing
$p(\mathbf{x_{S}})$ into conditionals. Inference from line 2 to line
3 of Eq. \ref{eq:general_ce_trans} is based on the results of Eq.
\ref{eq:ar_decom}. Tractable functions $\mathcal{F}_{\mathbf{S}}$
and $\mathcal{F}_{\mathbf{TS}}$ from trained MAFs, which are inherently
designed for easy partial derivative computation, enable efficient
estimation of $H\left(T|\mathbf{S}\right)$ via Eq. \ref{eq:general_ce_trans}
using Monte Carlo integration on $K$ i.i.d. observations uniformly
sampled from $\mathcal{D}$, where $\mathbf{x}^{(k)}=\left[x_{1}^{(k)},x_{2}^{(k)},\ldots,x_{d}^{(k)}\right]$
$\forall k\in[K]$. Accordingly,
\begin{equation}
H\left(T|\mathbf{S}\right)\approx\frac{1}{K}\sum_{k=1}^{K}\left(\sum_{X_{i}\in\mathbf{S}}\log\left|\frac{\partial f_{\mathbf{S}}}{\partial x_{i}^{(k)}}\right|-\sum_{X_{i}\in\mathbf{TZ}}\log\left|\frac{\partial f_{\mathbf{S}}}{\partial x_{i}^{(k)}}\right|\right)\label{eq:general_ce_est}
\end{equation}

To obtain $\left|\nicefrac{\partial f_{\mathbf{S}}}{\partial x_{i}}\right|$
in Eq. \ref{eq:general_ce_trans}, it can be deduced from the MAF
architecture that access to the univariate conditionals $p(x_{i}\mid\mathbf{x_{<i}})$
is necessary during the autoregressive decomposition via the chain
rule, with $\mathbf{x}_{\mathbf{<i}}$ being the realizations of $\mathbf{X_{<i}}\subset\mathbf{S}$
ordered prior to $X_{i}$ within a specific variable order: this term
essentially represents the residual generated by the change-of-variable
in Eq. \ref{eq:nf_formula}. It has been argued that training an order-agnostic
conditioner $c$ is theoretically feasible and can provide access
to all conditionals from any order \cite{germain2015made}, given
that each order does not shuffle between stacks in normalizing flows.
This scenario, however, is empirically inefficient, as stated by the
authors in \cite{shih2022training}; thus, it is evidently not employed
by typical MAFs \cite{papamakarios2017masked,de2020block,huang2018neural}
in density estimation, particularly in high-dimensional spaces. For
this reason, we focus on the default setting of these MAFs in \cite{papamakarios2017masked,de2020block,huang2018neural},
which preserves the original variable order of input $\pi(d)=\left[1,2,\ldots,d\right]$
during training. As a trade-off, this design inherently restricts
the model to a single order and can only derive conditionals $p(x_{i}\mid\mathbf{x_{<i}})$
from $\pi(d)$. 

Given a fixed order $\pi(d)$, an alternative solution is to obtain
conditionals not from the joint but separately from marginals corresponding
to subsets of $\mathbf{V}$ \cite{shih2022training}. An arbitrary
subset $\mathbf{S}$ can be rearranged to form a discontinuous sub-sequence
of $\pi(d)$, denoted as $\pi_{\mathbf{S}}$, from which the marginal
for $\mathbf{S}$ can be decomposed to derive $\left|\nicefrac{\partial f_{\mathbf{S}}}{\partial x_{i}}\right|$.
A single MAF can be adapted to align with this intuition through a
minor architectural modification: Given $\mathbf{S}$, we mask all
$X_{i}\not\in\mathbf{S}$ in input $\mathbf{x}$ to zeros, then feed
this into MADE conditioner to produce a chain of $c(\mathbf{x_{<i}})$.
We assume that zero-masked positions in $\pi(d)$ contribute nothing
to $c(\mathbf{x_{<i}})$, so this term embeds only information of
$\mathbf{X_{<i}}\subset\mathbf{S}$ that precedes $X_{i}$ in $\pi_{\mathbf{S}}$.
Nonetheless, MADE \cite{germain2015made} initially uses only a portion
of the weights of its neural network (NN) for the joint input. Depending
on this fixed pre-masking, this straightforward approach fails to
fully utilize complete NN for all marginals, as empirically demonstrated
in Fig. \ref{fig:ce_options} for typical MAFs \cite{huang2018neural,de2020block}.

\subsubsection{Efficient redesign of Any-Subset Masked Autoencoders}\label{subsect:asmaf}

Building on Subsection \ref{subsec:ce_diffeo}, we aim to enable the
conditioner $c$ to efficiently compute $c(\mathbf{x_{<i}})$ for
any $X_{i}\in\mathbf{S}$, where $\mathbf{S}\subseteq\mathbf{V}\setminus\{T\}$,
and $\mathbf{x_{<i}}$ represents the realizations of $\mathbf{X_{<i}}$
preceding $X_{i}$ under order $\pi(d)$. For a sample $\mathbf{x}^{(k)}$
and subset $\mathbf{S}$, let $x_{\mathbf{S}}^{(k)}=\left[x_{i_{1}}^{(k)},\ldots,x_{i_{|\mathbf{S}|}}^{(k)}\right]$
denote the non-zero elements in the masked vector (with zero values
for $X_{i}\not\in\mathbf{S}$). We achieve this goal using a modified
version of MADE capable of processing any $x_{\mathbf{S}}^{(k)}$,
efficiently allocating neural autoencoder weights, propagating $x_{\mathbf{S}}^{(k)}$
through this weight region, and outputting $|\mathbf{S}|$ autoregressive
terms: ${\small c(\emptyset)}$, ${\small c\left(x_{i_{1}}^{(k)}\right)}$,...,
${\small c\left(x_{i_{1}}^{(k)},\ldots,x_{i_{|\mathbf{S}|-1}}^{(k)}\right)}$.
\textbf{\textit{RQ2.}} How can this modification be implemented?

\textbf{FANS} employs a robust dynamic weight masking mechanism, wherein
the weight regions of the autoencoder are adaptively assigned to each
subset $\mathbf{S}$. We draw on concepts from the Lottery Ticket
Hypothesis \cite{frankle2018lottery,chen2023structured} and various
pruning strategies, which identify a subset of valuable model connections:
A masked neural network with more connections demonstrates greater
expressiveness, enabling improved and faster learning from data. Intuitively,
we establish two objectives for our weight assignment/masking strategy:
(1) to greedily maximize the number of connections utilized across
the entire neural autoencoder for all subsets $\mathbf{S}$, and (2)
to ensure that the overlap between two subsets results in a greater
shared connection in their respective masks/weight regions. As illustrated
in Fig. \ref{fig:asmaf}, \textbf{FANS} masking mechanism operates
through two steps:\textbf{\textit{ Step 1.}}\textit{ Locate the fully-connected
weight region}: For subset $\mathbf{S}$, we redefine $\pi_{\mathbf{S}}=\left[i_{1},i_{2},\ldots,i_{|\mathbf{S}|}\right]$,
where each $i_{j}\in\left[1,d\right]$, as its increasing list of
variable positions---a discontinuous subsequence of $\pi(d)$. Each
hidden node in layer $l$ is assigned a score $s\in\left[1,d-1\right]$,
with selected nodes having scores $s\in\pi_{\mathbf{S}}\backslash\{i_{|\mathbf{S}|}\}$.
Similarly, nodes in the input and output layers have scores $s\in\left[1,d\right]$,
with selected nodes having scores $s\in\pi_{\mathbf{S}}$. We only
select connections between these chosen nodes across layers.\textit{
}\textbf{\textsl{Step 2}}\textbf{\textit{.}}\textit{ Force autoregressive
masking}: The fully-connected weight region between two consecutive
layers is masked autoregressively, similar to MADE, ensuring that
a higher-layer node (considered from input to output) with score $s_{h}$
receives information from lower-layer nodes with scores $s_{l}\leq s_{h}$.
The exception is from the highest hidden layer to the output layer,
where $s_{l}<s_{h}$. 

\textbf{\textsl{Enabling compactness.}} For the MB discovery task,
the MB size is often considerably smaller than the total number of
variables, particularly in large sparse causal networks, leading to
the intuition that investigating only subsets $\mathbf{S}$ with maximum
size $\max|\mathbf{S}|\ll d$ is sufficient. Let $M$ denote a hyperparameter
related to the maximum size of observed subsets. Each hidden layer
of the \textbf{compacted FANS} requires approximately $\alpha_{B}(M-1)$
nodes instead of $\alpha_{B}(d-1)$ nodes, where $\alpha_{B}$ is
the number of blocks in each hidden layer (see Fig. \ref{fig:asmaf}
for details), due to the reduced size of $\pi_{\mathbf{S}}$. With
$\pi_{\mathbf{S}}=\left[i_{1},i_{2},\ldots,i_{|\mathbf{S}|}\right]$,
each hidden node from \textsl{Step 1} would have a score $s$ that
satisfies $s\in[1,M-1]\cap\pi_{\mathbf{S}}^{\%}$, where $\pi_{\mathbf{S}}^{\%}=\left[\left\lceil \frac{(M-1)i_{1}}{d-1}\right\rceil ,\ldots,\left\lceil \frac{(M-1)i_{|\mathbf{S}|}}{d-1}\right\rceil \right]$
is the rescaled counterpart of $\pi_{\mathbf{S}}$ from $d$-scale
to $M$-scale. Rescaling may result in overlapping elements in $\pi_{\mathbf{S}}^{\%}$,
leading to $|\pi_{\mathbf{S}}^{\%}|<|\mathbf{S}|-1$, while at least
$|\mathbf{S}|-1$ nodes are needed within each block of a hidden layer
to carry autoregressive information $(x_{i_{1}})$, $(x_{i_{1}},x_{i_{2}})$,...,
$(x_{i_{1}},\ldots,x_{|\mathbf{S}|-1})$. These colliders can be addressed
by greedily adjusting element values in $\pi_{\mathbf{S}}^{\%}$ to
fill possible gaps between them. For example, given $M-1=6$ and $\pi_{\mathbf{S}}^{\%}=[1,1,2,6,6]$,
the values in $\pi_{\mathbf{S}}^{\%}$ are adjusted from both sides
as follows: $[1,1,2,6,6]\rightarrow[1,2,3,6,6]\rightarrow[1,2,3,5,6]$.
The efficacy of the compact version is demonstrated in Fig. \ref{fig:ce_options}.

\textbf{\textit{Training objective with FANS masking mechanism.}}
In alignment with the training objective for density estimation in
other MAF models, \textbf{FANS} is trained by maximizing the total
log likelihood on training data, effectively fitting $p_{\mathbf{}}(\mathbf{x})$
through the diffeomorphic transformation in Eq. \ref{eq:nf_formula}.
The base distribution is chosen to be standard Gaussian. Notably,
data samples are stochastically masked using a uniform $d$-dimensional
binary mask distribution to fit marginals corresponding to subsets
$\mathbf{S}$. Let $\mathcal{M}$ denote the uniform mask distribution,
$M_{\mathbf{S}}\sim\text{\ensuremath{\mathcal{M}}}$ represent the
binary mask for $\mathbf{S}$, and $\mathbf{x}_{M_{\mathbf{S}}}:=M_{\mathbf{S}}\odot\mathbf{x}$.
We maximize the log likelihood of this term using training data in
$\mathcal{D}$:
\begin{align}
\mathbb{E}_{\substack{\mathbf{x\sim}p(\mathbf{x})\\
M_{\mathbf{S}}\sim\mathcal{M}
}
}\left[\log p_{\mathbf{u}}\left(f_{\theta}(\mathbf{x}_{M_{\mathbf{S}}})\odot M_{\mathbf{S}}\right)+\sum_{X_{i}\in\mathbf{S}}\log\left|\frac{\partial f_{\theta}(\mathbf{x}_{M_{\mathbf{S}}})}{\partial x_{i}}\right|\right]
\end{align}

\begin{rem}
Sampling a mask $M_{\mathbf{S}}$ for $\mathbf{S}$ to fit marginal
$p(\mathbf{x_{S}})$ also implicitly fits marginals $p(x_{\mathbf{S_{-1}}})$,
$p(x_{\mathbf{S_{-2}}})$, ... of its prefix subsets $\mathbf{S}_{-1}$,
$\mathbf{S}_{-2}$, ..., where $\pi_{\mathbf{S}}=\left[i_{1},\ldots,i_{|\mathbf{S}|}\right]$
and $\pi_{\mathbf{S}_{-k}}=\left[i_{1},\ldots,i_{|\mathbf{S}|-k}\right]$.
This is due to $p(\mathbf{x_{S}})=p(\mathbf{x_{S_{-1}}})p(x_{i_{|\mathbf{S}|}}|\mathbf{x_{S_{-1}}})=p(\mathbf{x_{S_{-2}}})p(x_{i_{|\mathbf{S}|-1}}|\mathbf{x_{S_{-2}}})p(x_{i_{|\mathbf{S}|}}|\mathbf{x_{S_{-1}}})=\ldots$
Define a \textit{leaf-subset}\textbf{ }$\mathbf{S}_{\text{leaf}}$
as a subset that is not a prefix of any other subset except itself
(e.g., leaf node $[3,4]$ in Fig. 2(c) of \cite{shih2022training}).
Uniformly sampling \textit{leaf-masks} $M_{\mathbf{S}_{\text{leaf}}}$
is sufficient, reducing the combinatorial space of masks by half (empirically
shown in Fig. \ref{fig:ce_options}). Additionally, partially observing
subsets $\mathbf{S}$ of size $|\mathbf{S}|\leq M\ll d$ is sufficient
for sparse graphs. Combining these insights, our masking strategy
efficiently operates within a small fraction of the exponential masking
space.
\end{rem}

\subsection{Polynomial-time greedy minimization}\label{subsec:digmb}

Optimizing objectives with set-based solutions is an NP-hard combinatorial
problem \cite{chen2015sequential,das2011submodular}. To address this
in polynomial time, two common approaches are (1) continuous relaxation
\cite{wu2023practical}, which learns a differentiable subset mask,
and (2) greedy search \cite{das2011submodular,chen2015sequential,chen2018weakly,sharma2015greedy},
which iteratively refines the predictive set. We adopt the latter
due to the complex implementation of continuous relaxation in our
high-dimensional settings and the lack of explicit control from challenges
in regularization parameter tuning \cite{sharma2015greedy,das2011submodular}.
Accordingly, we adopt a greedy approach, which is simple yet efficient.
We reformulate MB discovery problem as a \textbf{\textit{sequential
information minimization}} task \cite{chen2015sequential,das2011submodular},
a framework demonstrated to perform well within information-theoretic
learning paradigms and is theoretically near-optimal under certain
noise assumptions \cite{chen2015sequential,sharma2015greedy}. Thus,
it provides a solid foundation for our error-bound theorem (see our
Theorem \ref{thm:Bound}).

We introduce \textbf{FANSITEMB} (Algorithm \ref{algo:naigmb}), which
guarantees convergence to a minimal set $\mathbf{S}^{*}$ predicting
the true $\mathcal{MB}_{X_{i}}$ by minimizing $H(X_{i}\mid\mathbf{S})$
(see our Theorem \ref{thm:soundness-cemb}). In the optimal case,
the \textit{growing phase} can yield the set $\tilde{\mathcal{MB}}_{X_{i}}^{+}$
satisfying $\mathcal{MB}_{X_{i}}\subseteq\tilde{\mathcal{MB}}_{X_{i}}^{+}$,
where $H(X_{i}\mid\tilde{\mathcal{MB}}_{X_{i}}^{+})=H(X_{i}\mid\mathcal{MB}_{X_{i}})=\min_{\mathbf{S}\subseteq\mathbf{V}\setminus\{X_{i}\}}H(X_{i}\mid\mathbf{S})$
(Theorem \ref{thm:min_hcond}). The \textit{shrinking phase} then
removes redundant variables from $\tilde{\mathcal{MB}}_{X_{i}}^{+}$
with minimal impact on $H(X_{i}\mid\tilde{\mathcal{MB}}_{X_{i}}^{+})$.
For discovering MBs of all variables in a causal network, \textbf{FANSITEMB}
can operate in a \textit{distributed parallel setup}, running independently
for each target variable, followed by symmetry correction: $X_{i}\in\mathcal{MB}_{X_{j}}\Leftrightarrow X_{j}\in\mathcal{MB}_{X_{i}}$.
Additionally, it can exploit \textit{GPU parallelism} by batch-computing
$H(X_{i}|\mathbf{S}\cup\{X_{j}\})$ for all candidate variables $X_{j}$
during each step of the growing and shrinking phases. This ensures
a computational complexity of $\mathcal{O}(bM)$ for inferring MBs
of all variables simultaneously, where $b$ is the batch count and
$M$ is the estimated maximum size of an MB, as defined. Compared
to the closely related KMB \cite{wu2023practical}, which employs
an iterative MB discovery strategy and requires solving its training
objective at least $M$ times per target $X_{i}$, \textbf{FANSITEMB}
is significantly more efficient.
\begin{thm}[Error bound of \textbf{FANSITEMB}]
\label{thm:Bound}Given that $\mathcal{F}_{\theta}^{*}$ is a universal
approximator in the model class of \textbf{FANS}, it can approximate
any conditional entropy $H(X_{i}\mid\mathbf{S})$ from Eq. \ref{eq:general_ce_trans}
with arbitrary precision, where $X_{i}\in\mathbf{V}$ and $\mathbf{S}\subseteq\mathbf{V}\backslash\{X_{i}\}$,
with $\mathbf{V}$ being the set of all variables under a faithful
distribution. Assuming the growing phase of Algorithm \ref{algo:naigmb}
stops at a subset $\tilde{\mathcal{MB}}_{X_{i}}^{+}\subset\mathbf{V}\backslash\{X_{i}\}$
where $\lvert\tilde{\mathcal{MB}}_{X_{i}}^{+}\rvert=k'$, let $\mathcal{MB}_{X_{i}}$
denote the true MB of $X_{i}$, referred to as the optimal solution
of Algorithm \ref{algo:naigmb}, with $\lvert\mathcal{MB}_{X_{i}}\rvert=k$.
In the general case, for any $\delta>0$, an estimation bound can
be directly derived from Theorem 2 in \cite{chen2015sequential}:
\[
0\leq H(X_{i}\mid\tilde{\mathcal{MB}}_{X_{i}}^{+})-H(X_{i}\mid\mathcal{MB}_{X_{i}})\leq\delta_{e}+\Delta_{N}
\]

where $\Delta_{N}:=\log N\left[1-\exp\left(-\frac{k'}{k\gamma\max\{\log N,\log(1/\delta_{e})\}}\right)\right]$,
and $\gamma$ is a constant that depends on the measurement noise
from $N$ samples. In the case of linear Gaussian additive noise,
we can also derive a bound on the ratio between exponential terms
as follows:
\[
1\leq\frac{2\pi e\sigma_{X_{i}}^{2}-e^{2H(X_{i}\mid\mathcal{MB}_{X_{i}})}}{2\pi e\sigma_{X_{i}}^{2}-e^{2H(X_{i}\mid\tilde{\mathcal{MB}}_{X_{i}}^{+})}}\leq\frac{1}{1-e^{-\lambda_{\min}(\mathbf{C},2k')}}
\]

where $\sigma_{X_{i}}$ is the standard deviation of $X_{i}$ and
$\lambda_{\min}(\mathbf{C},2k')$ is the smallest eigenvalue among
all $2k'\times2k'$ sub-matrices of the joint covariance matrix after
standardization, i.e., the correlation matrix $\mathbf{C}$. 
\end{thm}
We provide the proof in the Appendix \ref{subsec:proof_b}.

If $k'\geq k\gamma\max\left\{ \log N,\log\frac{1}{\delta_{e}}\right\} \ln\left(\frac{\log N}{\delta_{e}}\right)$,
then $H(X_{i}\mid\tilde{\mathcal{MB}}_{X_{i}}^{+})$ $-$ $H(X_{i}\mid\mathcal{MB}_{X_{i}})\leq2\delta_{e}$
\cite{chen2015sequential}. Hence, up to $\delta_{e}$ in absolute
terms, we can approach the optimal MB achievable within $k=\lvert\mathcal{MB}_{X_{i}}\rvert$
steps by expanding $\tilde{\mathcal{MB}}_{X_{i}}^{+}$ sufficiently,
within a logarithmic factor of $k$. In the linear Gaussian case,
additional insights emerge: for sparse networks, where variables exhibit
weaker correlations or independence, $\lambda_{\min}(\mathbf{C},2k')$
tends to be larger, yielding a tighter upper bound.
\begin{thm}[Soundness of \textbf{FANSITEMB}]
\label{thm:soundness-cemb}Under faithfulness, causal sufficiency,
and $\mathcal{F}_{\theta}^{*}$ as defined in Theorem \ref{thm:Bound},
Algorithm \ref{algo:naigmb} discovers all and only variables in the
MB of the target node.
\end{thm}
The proof is provided in the Appendix \ref{subsec:proof_c}.
\begin{algorithm}[h]
\caption{\textbf{FANSITEMB} for a single target}\label{algo:naigmb}
\begin{algorithmic}[1]
	\renewcommand{\algorithmicrequire}{\textbf{Input:}}
	\renewcommand{\algorithmicensure}{\textbf{Output:}}

	\REQUIRE Target $X_{i}\in\mathbf{V}$, $K$ i.i.d. data samples in $\mathcal{D}_{K}\subseteq\mathcal{D}$, trained approximator $f_{\theta}$, subset size threshold $M$, growing stop threshold $\epsilon_{g}$, shrinking stop threshold $\epsilon_{s}$, patience coefficient $\rho$. 

	\ENSURE Predicted MB $\tilde{\mathcal{MB}}_{X_{i}}$.

\STATE Initialize $\tilde{\mathcal{MB}}_{X_{i}} \gets \emptyset$, patience $\gets 0$. 

\textbf{Growing phase}
\WHILE{exists $X_j \in \mathbf{V} \setminus (\{X_i\} \cup \tilde{\mathcal{MB}}_{X_i})$ \AND $\text{patience}\leq\rho$}
	\STATE Let $X_{m}\in \mathbf{V} \setminus (\{X_i\} \cup \tilde{\mathcal{MB}}_{X_i})$ be a variable minimizing $H(X_{i}\mid\tilde{\mathcal{MB}}_{X_{i}} \cup \{X_{m}\})$. \COMMENT{$\triangleright$ $H(\cdot\mid\cdot)$ is estimated via Eq. ~\ref{eq:general_ce_est} or Eq. ~\ref{eq:gauss_ce_trans}}

\IF {$H(X_{i} \mid \tilde{\mathcal{MB}}_{X_{i}})-H(X_{i}\mid\tilde{\mathcal{MB}}_{X_{i}} \cup \{X_{m}\}) > \epsilon_{g}$}
  \STATE patience $\gets$ 0. 
\ELSE
  \STATE patience $\gets$ patience + 1.
\ENDIF
\STATE $\tilde{\mathcal{MB}}_{X_{i}} \gets \tilde{\mathcal{MB}}_{X_{i}} \cup \{X_{m}\}$.
\ENDWHILE 

\textbf{Shrinking phase}
\WHILE{$\tilde{\mathcal{MB}}_{X_{i}} \neq \emptyset$}
	\STATE Let $X_{m}\in\tilde{\mathcal{MB}}_{X_{i}}$ be a variable minimizing $\delta_{H} = H(X_{i}\mid\tilde{\mathcal{MB}}_{X_{i}} \setminus \{X_{m}\}) - H(X_{i}\mid\tilde{\mathcal{MB}}_{X_{i}})$.
	\IF{$\delta_{H} > \epsilon_{s}$}
		\STATE \textbf{break}
	\ENDIF
	\STATE $\tilde{\mathcal{MB}}_{X_{i}} \gets \tilde{\mathcal{MB}}_{X_{i}} \setminus \{X_{m}\}$.
\ENDWHILE 
\end{algorithmic} 
\vspace{0.1cm}
\end{algorithm}

\subsection{Differential causal discovery (DCD) with minimal I-map-based initialization}

We investigate the integration of MB discovery with DCD, a recent
advancement in causal discovery (CD). Principally, MB discovery serves
as a preselection step, reducing the graph space to enhance numerical
stability and scalability in DAG estimation. SDCD \cite{nazaret2024stable}
and DAT \cite{amin2024scalable} employ sparse regression objectives
to learn MBs. DAT \cite{amin2024scalable} minimizes: $\mathcal{L}_{n}(\theta_{n})=\mathbb{E}\left[X^{n}-f_{\theta_{n}}(X^{m})_{m\neq n}\right]^{2}+\lambda\mathop{\sum|W_{i,j}^{n}|}$
while SDCD \cite{nazaret2024stable} optimizes $\mathcal{L}_{n}(\theta_{n})=-\mathbb{E}\left[p_{n}(x_{n}\mid x_{-n};\theta_{n})\right]+\lambda\mathop{\sum|W_{i,j}^{n}|}$
where $\lambda$ is the regularization term and $W_{i,j}$ denotes
adjacency matrix weights. The warm-up step of SDCD and DAT has key
limitations: (1) a single regularization term controls sparsity uniformly,
ignoring MB size variation (e.g., smaller MBs may need larger $\lambda$,
larger MBs require smaller $\lambda$); (2) SEM-derived weights poorly
reflect true dependencies, for which mutual information performs better,
leading to optimization errors that propagate in later DAG search.
Multiple configurations of $W_{i,j}^{n}$ can yield nearly identical
values of $\mathcal{L}_{n}(\theta_{n})$, reducing reliability. Building
upon SDCD \cite{nazaret2024stable}, we introduce \textbf{FANSITE-DCD},
which leverages \textbf{FANSITEMB}-inferred MBs to construct a moral
graph used as a mask. Masking on this moral graph allows the DAG search
in the later stage to continue removing redundant edges under DAG
constraints (e.g., eliminating reverse edges and spouse links). The
neural network and DAG penalty score from SDCD are utilized to support
this process. While \textbf{FANSITE-DCD} improves upon DCD, it can
also serve as a moral-graph-based initialization step for other causal
discovery methods.
\begin{defn}[Moral graph]
\label{def:moral}The moral graph $\mathcal{M}[\mathcal{G}]$ of
a Bayesian network $\mathcal{G}$ with variables in $\mathbf{V}$
and an underlying positive distribution $P$ is the undirected graph
over $\mathbf{V}$ containing edges between each variable $X_{i}$
and variables in its MB. $\mathcal{M}[\mathcal{G}]$ serves as the
\textit{unique minimal I-map} for $P$ \cite{koller2009probabilistic}. 
\end{defn}
\begin{rem}
An I-map for $P$ is any graph $\mathcal{K}$ with its set of conditional
independencies (CI) $\mathcal{I}(\mathcal{K})$ satisfying $\mathcal{I}(\mathcal{K})\subseteq\mathcal{I}$,
where $\mathcal{I}$ is the CI set associated with $P$ (i.e., $X \perp\!\!\!\perp Y \mid \mathbf{Z}$).
From Definition 3.13 in \cite{koller2009probabilistic} and Definition
\ref{def:moral}, we deduce that the moral graph $\mathcal{M}[\mathcal{G}]$
is the \textbf{\textit{sparsest I-map}}, with $\mathcal{I}(\mathcal{K})\subseteq\mathcal{I}(\mathcal{M}[\mathcal{G}])$
for all $\mathcal{K}$. Thus, if the MBs of all variables are accurately
recovered, the induced moral graph establishes a substantially reduced
space for DAG search, forming a \textit{minimal I-map-based subspace}. 
\end{rem}
\begin{prop}
Under the conditions stated in Theorem \ref{thm:soundness-cemb},
\textbf{FANSITE-DCD} utilizes observational data to search for the
optimal DAG within a minimal I-map-based subspace, rather than starting
from the larger space of all I-maps. This approach enables it to converge
more accurately and quickly to a DAG $\mathcal{\hat{G}}$ that is
I-Markov-equivalent to the true DAG $\mathcal{\mathcal{G}^{*}}$. 
\end{prop}
This proposition is direct facilitated by our Theorem \ref{thm:soundness-cemb},
following Definition \ref{def:moral} for moral graphs. 
\vspace{-0.15cm}

\section{Experiments\vspace{-0.1cm}}\label{sec:Results}

\subsection{Methods}

For \textbf{\textit{MB discovery}}, we evaluate: (1) constraint-based
approaches, including IAMB \cite{tsamardinos2003towards}, which dynamically
selects the most associated features to the target conditioned on
selected ones, and TC \cite{pellet2008using}, which prunes edges
by conditioning on all remaining variables; (2) score-based methods
such as S\textsuperscript{2}TMB \cite{gao2017efficient}, leveraging
the coexistence of spouses and descendants alongside decomposable
Bayesian scores, and kernel-based KMB \cite{wu2023practical}, minimizing
a conditional covariance operator; (3) Hybrid approaches like EEMB
\cite{wang2020towards}, combining CI tests for parent-child discovery
with S\textsuperscript{2}TMB scores for spouse discovery; and (4)
DAT-Moral, which serves as the MB preselection stage using sparse
regression in DAT \cite{amin2024scalable}. 

For \textbf{\textit{causal discovery}}, we assess three strategies:
(1) MB-informed constraint-based methods like MARVEL \cite{mokhtarian2025recursive},
using recursive learning via removable variables; (2) plain differentiable
causal discovery (DCD) such as DAGMA \cite{bello2022dagma}, which
employs log-determinant acyclicity constraint, and COSMO \cite{massidda2024constraint},
a constraint-free approach using an orientation matrix; and (3) MB-based
DCD identifying MBs in the preselection stage using sparse regression,
consisting of SDCD \cite{nazaret2024stable}, which uses a neural
autoencoder and minimizes negative log-likelihood in later DAG estimation,
while DAT\textbf{ }\cite{amin2024scalable} trains an amortized model
to learn separating sets and applies differentiable adjacency tests
for edge pruning. We adopt default configurations recommended by respective
baseline methods. For our approach, we utilize our class of masked
autoregressive flows \textbf{FANS} as the conditional entropy estimator
for nonlinear and real data, while using Eq. \ref{eq:gauss_ce_trans}
for linear Gaussian cases. 

\subsection{Datasets}

\subsubsection{Synthetic data}

The process for generating synthetic data strictly follows public
code from the \texttt{\textit{gcastle}} package.\footnote{https://github.com/huawei-noah/trustworthyAI}
Erdős-Rényi causal DAGs are generated with node counts $d\in\{30,100,1000,5000\}$
for both linear and nonlinear Gaussian SEMs, with weighted adjacency
matrices constructed using edge weights randomly sampled from $\mathcal{\mathcal{U}}\left([-2,-0.5]\cup[0.5,2]\right)$.
Linear data is generated using the linear SEM model $\mathbf{X}=\mathbf{W}^{\mathsf{T}}\mathbf{X}+\epsilon$,
where $\epsilon\sim\mathcal{N}(0,1)$, while nonlinear data is generated
following $X_{i}=f_{i}(X_{Pa_{i}})+\epsilon_{i}$, where $f_{i}$
is sampled from a Gaussian Process with an RBF kernel of bandwidth
$1$. For each setting, five datasets are sampled, each containing
the ground truth adjacency matrix of the causal DAG, a training dataset
of 1000 samples for $d<100$, and 5000 samples for $d\geq100$. All
evaluation scenarios include (1) linear SEMs: sparse graphs $\left\{ \text{d30},\text{d100-1},\text{d1000}\right\} $
with average degree $\bar{D}=1$, dense graphs $\text{d100-2}$ with
$\bar{D}=6$, and large graphs $\text{d5000}$ with $\bar{D}=1.5$,
and (2) nonlinear SEMs: $\left\{ \text{d30-G},\text{d100-1}\right\} $
with $\epsilon_{i}\sim\mathcal{N}(0,1)$, and $\text{d30-Mixed-noises}$
($\text{d30-MN}$) with $\epsilon_{i}\sim p_{\epsilon_{i}}$, where
$p_{\epsilon_{i}}\in\{\mathcal{N}(0,1),\text{\ensuremath{\mathcal{U}(-1,1)},\text{Laplace}}(0,1),\text{Gumbel}(0,1),\text{Exp}(0,1)\}$. 

\subsubsection{Real and semi-real networks}

We use four public real-world/semi-real gene regulatory networks (GRN):
(1) Sachs \cite{sachs2005causal}, a protein interaction network with
expression data for proteins and phospholipids in human cells; (2)
SynTReN \cite{van2006syntren}, the semi-synthetic GRN based on real
topologies and Michaelis-Menten/Hill kinetics; (3) SERGIO \cite{dibaeinia2020sergio},
the simulated gene expression dataset using high-throughput scRNA-seq
technologies guided by real GRNs (we use 3,000 single cells, three
bins each, and stochastic noise modeled via the Dual Production Decay
framework); and (4) ARTH150 \cite{opgen2007correlation}, an Arabidopsis
thaliana GRN from the \texttt{\textit{bnlearn}}{\ttfamily\footnote{https://www.bnlearn.com/}}
repository.\vspace{-0.7cm}
\begin{table}[h]
\caption{Real datasets}

\vspace{0.7cm}

\centering
\small 
    \renewcommand{\arraystretch}{0.8} % Tighten row spacing
    \begin{tabular}{@{}lccc@{}}
        \toprule
        {Dataset} & {Nodes} & {Edges} & {Observations} \\ \midrule
{\textbf{Sachs} \cite{sachs2005causal}} & {11} & {17} & {853} \\
{\textbf{SynTReN} \cite{van2006syntren}} & {20} & {24} & {500} \\
{\textbf{SERGIO} \cite{dibaeinia2020sergio}} & {100} & {137} & {9000} \\
{\textbf{ARTH150} \cite{opgen2007correlation}} & {107} & {150} & {5000} \\
		\bottomrule
    \end{tabular}
\end{table}
\vspace{-0.55cm}
\begin{table*}
\vspace{-0.5cm}

\caption{\textbf{Markov boundary discovery performance in \%}. The numbers
are \textit{mean \textpm{} standard} \textit{deviation} over 5 independent
simulations, except for real data. Missing entries ($-$) correspond
to instances with runtime exceeding the limit of 10 hours. \textbf{Bold}:
best performance, \textit{italic}: second-best performance.}\label{tab:mb_performance_metrics}

\vspace{0.7cm}

\resizebox{\linewidth}{!}{

\centering
	\small
	\renewcommand{\arraystretch}{0.8}
    \begin{tabular}{@{}lcccccccc@{}}
	
	\toprule
        {Graph type} & {Metric} & \textbf{IAMB} \cite{tsamardinos2003towards} & \textbf{S2TMB} \cite{gao2017efficient} & \textbf{EEMB} \cite{wang2020towards} & \textbf{TC} \cite{pellet2008using} & \textbf{KMB} \cite{wu2023practical} & \textbf{DAT-Moral} \cite{amin2024scalable} & \textbf{FANSITEMB (ours)} \\ \midrule

         \multicolumn{9}{c}{\textbf{Linear data}} \\ \midrule

\multirow{3}{*}{d30} & nDCG ($\uparrow$) & 47.18 $\pm$ 12.81 & 93.01 $\pm$ 3.74 & 90.78 $\pm$ 6.54 & \underline{98.91 $\pm$ 1.25} & 70.58 $\pm$ 7.60 & 90.00 $\pm$ 3.84 & \textbf{99.34 $\pm$ 1.02} \\

        & AveP ($\uparrow$) & 40.34 $\pm$ 15.41 & 90.56 $\pm$ 5.11 & 87.42 $\pm$ 8.93 & \underline{98.42 $\pm$ 1.81} & 60.69 $\pm$ 9.57 & 86.14 $\pm$ 5.17 & \textbf{99.08 $\pm$ 1.39} \\

        & F1 ($\uparrow$) & 48.25 $\pm$ 12.72 & 93.58 $\pm$ 3.79 & 91.39 $\pm$ 6.40 & \underline{98.68 $\pm$ 1.72} & 65.36 $\pm$ 7.04 & 21.28 $\pm$ 3.22 & \textbf{99.44 $\pm$ 0.87} \\ \midrule

      \multirow{3}{*}{d100-1} & nDCG ($\uparrow$) & 41.71 $\pm$ 5.36 & 94.27 $\pm$ 1.90 & 96.01 $\pm$ 2.00 & \textbf{99.90 $\pm$ 0.22} & {$-$} &  97.42 $\pm$ 1.34 & \underline{99.86 $\pm$ 0.22} \\

        & AveP ($\uparrow$) & 34.54 $\pm$ 5.37 & 91.80 $\pm$ 2.74 & 94.29 $\pm$ 2.88 & \textbf{99.85 $\pm$ 0.34} & {$-$} & 96.33 $\pm$ 1.72 & \underline{99.78 $\pm$ 0.33} \\

        & F1 ($\uparrow$) & 42.87 $\pm$ 5.64 & 95.60 $\pm$ 1.73 & 96.64 $\pm$ 1.91 & \underline{98.59 $\pm$ 0.60} & {$-$} & 22.33 $\pm$ 1.14 & \textbf{99.88 $\pm$ 0.18} \\ \midrule

      \multirow{3}{*}{d100-2} & nDCG ($\uparrow$) & 38.69 $\pm$ 3.11 & 42.69 $\pm$ 2.74 & 20.13 $\pm$ 1.05 & \underline{83.43 $\pm$ 9.49} & {$-$} &  45.58 $\pm$ 6.08 & \textbf{97.12 $\pm$ 0.34} \\

        & AveP ($\uparrow$) & 19.81 $\pm$ 2.60 & 26.33 $\pm$ 1.83 & 8.07 $\pm$ 0.51 & \underline{75.17 $\pm$ 13.92} & {$-$} & 25.82 $\pm$ 6.01 & \textbf{95.65 $\pm$ 0.54} \\

        & F1 ($\uparrow$) & 40.99 $\pm$ 2.12 & 46.39 $\pm$ 2.43 & 19.27 $\pm$ 0.88 & \underline{78.09 $\pm$ 10.98} & {$-$} & 57.01 $\pm$ 2.78 & \textbf{97.38 $\pm$ 0.32} \\ \midrule

      \multirow{3}{*}{d1000} & nDCG ($\uparrow$) & 43.53 $\pm$ 1.79 & 96.20 $\pm$ 0.59 & 96.45 $\pm$ 0.33 & \textbf{99.98 $\pm$ 0.02} & {$-$} & 99.10 $\pm$ 0.28 & \underline{99.96 $\pm$ 0.05} \\

        & AveP ($\uparrow$) & 36.58 $\pm$ 1.95 & 94.90 $\pm$ 0.78 &  94.93 $\pm$ 0.47 & \textbf{99.97 $\pm$ 0.03} & {$-$} & 98.64 $\pm$ 0.42 & \underline{99.93 $\pm$ 0.07} \\

        & F1 ($\uparrow$) & 45.50 $\pm$ 1.71 & \underline{96.82 $\pm$ 0.49} & 96.76 $\pm$ 0.30 & 87.06 $\pm$ 1.40 & {$-$} & 45.05 $\pm$ 1.62 & \textbf{99.96 $\pm$ 0.04} \\ \midrule

      \multirow{3}{*}{d5000} & nDCG ($\uparrow$) & {$-$} & \underline{88.11 $\pm$ 1.08} & 87.8 $\pm$ 1.13 & 61.58 $\pm$ 0.72 & {$-$} & {$-$} & \textbf{99.59 $\pm$ 0.09} \\

        & AveP ($\uparrow$) & {$-$} & \underline{83.87 $\pm$ 1.43} & 82.89 $\pm$ 1.52 & 50.73 $\pm$ 0.77 & {$-$} & {$-$} & \textbf{99.37 $\pm$ 0.14} \\

        & F1 ($\uparrow$) & {$-$} & \underline{89.75 $\pm$ 0.95} & 88.65 $\pm$ 1.06 & 52.89 $\pm$ 0.76 & {$-$} & {$-$} & \textbf{99.61 $\pm$ 0.09} \\ \midrule

\multicolumn{9}{c}{\textbf{Nonlinear data}} \\ \midrule

\multirow{3}{*}{d30-G} & nDCG ($\uparrow$) & 29.25 $\pm$ 10.20 & 56.76 $\pm$ 9.78 & 56.57 $\pm$ 10.14 & 55.67 $\pm$ 10.22 & 68.50 $\pm$ 7.43 & \underline{80.88 $\pm$ 4.59} & \textbf{93.5 $\pm$ 1.63} \\

        & AveP ($\uparrow$) & 22.81 $\pm$ 9.33 & 48.28 $\pm$ 10.30 & 48.05 $\pm$ 10.83 & 47.23 $\pm$ 11.29 & 59.00 $\pm$ 8.26 & \underline{74.11 $\pm$ 5.77} & \textbf{90.58 $\pm$ 2.23} \\

        & F1 ($\uparrow$) & 30.49 $\pm$ 9.82 & 58.8 $\pm$ 9.20 & 59.06 $\pm$ 9.65 & 55.52 $\pm$ 10.05 & \underline{61.84 $\pm$ 8.26} & 21.28 $\pm$ 3.22 & \textbf{91.2 $\pm$ 2.08} \\ \midrule

\multirow{3}{*}{d30-MN} & nDCG ($\uparrow$) & 27.33 $\pm$ 7.26 & 52.24 $\pm$ 2.75 & 54.44 $\pm$ 3.07 & 53.13 $\pm$ 4.32 & 61.66 $\pm$ 5.50 & \underline{76.60 $\pm$ 7.45} & \textbf{90.67 $\pm$ 4.85} \\

        & AveP ($\uparrow$) & 20.98 $\pm$ 4.46 & 43.63 $\pm$ 3.80 & 45.94 $\pm$ 4.53 & 45.20 $\pm$ 5.39 & 51.69 $\pm$ 5.80 & \underline{69.01 $\pm$ 8.23} & \textbf{86.56 $\pm$ 6.44} \\

        & F1 ($\uparrow$) & 28.19 $\pm$ 9.12 & 54.99 $\pm$ 2.58 & \underline{56.09 $\pm$ 2.95} & 52.66 $\pm$ 3.93 & 55.32 $\pm$ 5.72 & 21.28 $\pm$ 3.22 & \textbf{87.72 $\pm$ 6.12} \\ \midrule

\multirow{3}{*}{d100-1} & nDCG ($\uparrow$) & 33.07 $\pm$ 4.80 & 65.21 $\pm$ 3.64 & 65.28 $\pm$ 3.42 & 70.83 $\pm$ 3.68 & {$-$} & \underline{87.91 $\pm$ 5.20} & \textbf{90.91 $\pm$ 2.41} \\

        & AveP ($\uparrow$) & 26.41 $\pm$ 3.83 & 56.61 $\pm$ 3.42 & 56.71 $\pm$ 3.22 & 62.83 $\pm$ 4.01 & {$-$} & \underline{83.20 $\pm$ 6.97} & \textbf{86.72 $\pm$ 3.51} \\

        & F1 ($\uparrow$) & 34.02 $\pm$ 5.64 & 67.90 $\pm$ 3.71 & \underline{68.03 $\pm$ 3.28} & 67.10 $\pm$ 3.27 & {$-$} & 59.00 $\pm$ 1.04 & \textbf{87.54 $\pm$ 3.00} \\ \midrule

\multicolumn{9}{c}{\textbf{Real/Semi-real networks}} \\ \midrule

\multirow{3}{*}{Sachs} & nDCG ($\uparrow$) & 25.04 & 60.84 & 60.84 & 64.36 & \underline{73.22} & 70.2 & \textbf{86.44} \\

        & AveP ($\uparrow$) & 17.2 & 49.42 & 49.42 & 53.96 & \underline{64.2} & 60.08 & \textbf{80.44} \\

        & F1 ($\uparrow$) & 21.42 & 61.05 & 61.05 & 64.08 & \underline{64.68} & 58.17 & \textbf{84.48} \\ \midrule

\multirow{3}{*}{SynTReN} & nDCG ($\uparrow$) & 16.13 & 31.72 & 29.3 & 22.79 & 22.27 & \underline{33.96} & \textbf{68.36} \\

        & AveP ($\uparrow$) & 11.08 & 28.1 & 26.22 & 20.16 & 21.3 & \underline{31.59} & \textbf{64.39} \\

        & F1 ($\uparrow$) & 9.52 & 31.35 & 31.48 & 31.31 & 21.67 & \underline{33.03} & \textbf{68.18} \\ \midrule

\multirow{3}{*}{SERGIO} & nDCG ($\uparrow$) & \underline{45.03} & {$-$} & {$-$} & 38.11 & {$-$} & 14.70 & \textbf{85.93} \\

        & AveP ($\uparrow$) & \underline{42.42} & {$-$} & {$-$} & 35.01 & {$-$} & 12.63 & \textbf{82.56} \\

        & F1 ($\uparrow$) & \underline{33.75} & {$-$} & {$-$} & 14.03 & {$-$} & 5.51 & \textbf{84.16} \\ \midrule

\multirow{3}{*}{ARTH150} & nDCG ($\uparrow$) & 35.56 & 93.58 & \underline{94.12} & 92.55 & {$-$} & 79.73 & \textbf{96.66} \\

        & AveP ($\uparrow$) & 31.5 & 91.27 & \underline{92.19} & 91.36 & {$-$} & 76.72 & \textbf{95.97} \\

        & F1 ($\uparrow$) & 36.32 & \underline{94.54} & 94.52 & 90.49 & {$-$} & 38.48 & \textbf{96.12} \\ 
	\bottomrule
\end{tabular}

}

\vspace{-0.5cm}
\end{table*}
\begin{table*}
\caption{\textbf{Causal discovery performance}. The numbers are \textit{mean
\textpm{} standard} \textit{deviation} over 5 independent simulations,
except for real data. \textbf{Bold}: best performance, \textit{italic}:
second-best performance.}\label{tab:cd_performance_metrics}
\vspace{0.7cm}

\resizebox{\linewidth}{!}{

 \centering
	\small
	\renewcommand{\arraystretch}{0.8}
    \begin{tabular}{@{}lccccccc@{}}
    \toprule
        {Graph type} & {Metric} & \textbf{DAGMA} \cite{bello2022dagma} & \textbf{COSMO} \cite{massidda2024constraint} & \textbf{SDCD} \cite{nazaret2024stable} & \textbf{DAT} \cite{amin2024scalable} & \textbf{MARVEL} \cite{mokhtarian2025recursive} & \textbf{FANSITE-DCD (ours)} \\ \midrule

         \multicolumn{8}{c}{\textbf{Nonlinear data}} \\ \midrule

\multirow{3}{*}{d30-G} & {SHD ($\downarrow$)} & 15.4 $\pm$ 8.85 & 25 $\pm$ 5.66 & \underline{13 $\pm$ 5.43} & 420.2 $\pm$ 7.50 &	24.6 $\pm$ 6.23 &	\textbf{12.8 $\pm$ 6.53} \\

        & {AUC-ROC ($\uparrow$)} & 0.752 $\pm$ 0.095 &	0.822 $\pm$ 0.084 &	\underline{0.858 $\pm$ 0.074} &	0.502 $\pm$ 0.068 &	0.603 $\pm$ 0.039 &	\textbf{0.902 $\pm$ 0.045} \\

        & {AUC-PR ($\uparrow$)} & 0.488 $\pm$ 0.172 &	0.375 $\pm$ 0.077 &	\underline{0.570 $\pm$ 0.144} &	0.036 $\pm$ 0.013 &	0.144 $\pm$ 0.061 &	\textbf{0.592 $\pm$ 0.098} \\ \midrule

     \multirow{3}{*}{d30-MN} & {SHD ($\downarrow$)} & \underline{21.8 $\pm$ 7.82} &	48.4 $\pm$ 12.46 &	30.6 $\pm$ 9.81 &	419.8 $\pm$ 6.10 &	26.6 $\pm$ 8.91 &	\textbf{18.2 $\pm$ 8.14} \\

& {AUC-ROC ($\uparrow$)} & 0.654 $\pm$ 0.066 &	 \underline{0.777 $\pm$ 0.058} &	 0.766 $\pm$ 0.037 &	0.488 $\pm$ 0.061 &	0.578 $\pm$ 0.041 &	\textbf{0.866 $\pm$ 0.044} \\

        & {AUC-PR ($\uparrow$)} & 0.252 $\pm$ 0.118 &	0.207 $\pm$ 0.071 &	\underline{0.259 $\pm$ 0.065} &	0.034 $\pm$ 0.011 &	0.113 $\pm$ 0.062 &	\textbf{0.466 $\pm$ 0.121} \\ \midrule

\multirow{3}{*}{d100-1} & {SHD ($\downarrow$)} & 58.8 $\pm$ 7.46 &	39.4 $\pm$ 4.98 &	\underline{38 $\pm$ 6.24} &	80.6 $\pm$ 9.26 &	86.8 $\pm$ 10.31 &	\textbf{25.6 $\pm$ 2.3} \\

 & {AUC-ROC ($\uparrow$)} & 0.712 $\pm$ 0.023 &	0.815 $\pm$ 0.035 &	\underline{0.822 $\pm$ 0.044} &	 0.652 $\pm$ 0.027 &	0.604 $\pm$ 0.029 &	\textbf{0.920 $\pm$ 0.021} \\

        & {AUC-PR ($\uparrow$)} & 0.386 $\pm$ 0.042 &	\underline{0.592 $\pm$ 0.069} &	0.590 $\pm$ 0.102 &	0.118 $\pm$ 0.033 &	0.113 $\pm$ 0.039 &	\textbf{0.694 $\pm$ 0.054} \\ \midrule

\multicolumn{8}{c}{\textbf{Real/Semi-real networks}} \\ \midrule

\multirow{3}{*}{Sachs} & {SHD ($\downarrow$)} & 14 &	15 &	\underline{13} &	37 &	16 &	\textbf{10} \\

 & {AUC-ROC ($\uparrow$)} & 0.57 &	0.58 &	0.63 &	\underline{0.65} &	0.52 &	\textbf{0.73} \\

        & {AUC-PR ($\uparrow$)} & 0.2 &	0.22 &	\underline{0.28} &	0.2 &	0.16 &	\textbf{0.45} \\ \midrule

\multirow{3}{*}{SynTReN} & {SHD ($\downarrow$)} & \underline{28} &	36 &	37 &	32 &	29 &	\textbf{19} \\

 & {AUC-ROC ($\uparrow$)} & 0.59 &	\underline{0.65} &	0.6 &	0.53 &	0.51 &	\textbf{0.82} \\

        & {AUC-PR ($\uparrow$)} & 0.12 &	\underline{0.14} &	0.1 &	0.07 &	0.06 &	\textbf{0.42} \\ 
   
\bottomrule
\end{tabular}

}
\end{table*}

\subsection{Evaluation metrics}

\textbf{\textsl{MB evaluation: }}In addition to F1 score, two ranking
metrics from information retrieval are employed: (1) nDCG (Normalized
Discounted Cumulative Gain), and (2) AveP (Average Precision). We
observe that most surveyed MB discovery methods evolve their predicted
MBs by approaching variables from most to least probable. To facilitate
this, we do not alter the order of variables in the MB when running
the compared algorithms. In principle, a higher nDCG/AveP is reported
when variables in the true MB are positioned closer to the top of
the predicted list. We report the average results of F1, nDCG, and
AveP across all MBs of variables in the network\textbf{\textsl{. }}

\textbf{\textsl{Causal discovery thresholding \& evaluation:}}\textbf{
}For DCD methods, we follow standard CD practices by setting a final
threshold on post-training score adjacency matrices to ensure DAG
outputs \cite{massidda2024constraint,bello2022dagma,nazaret2024stable}.
We utilize three metrics that compare the estimated DAG with the
ground truth: (1) AUCPR (area under the curve of precision-recall),
(2) AUC-ROC (area under the receiver operating characteristic curve),
and (3) SHD (Structural Hamming Distance), which refers to the minimum
number of edge additions, deletions, and reversals required to transform
the recovered DAG into the true DAG.

\subsection{Main results}

Effectiveness of proposed methods \textbf{FANSITEMB} and \textbf{FANSITE-DCD}
across various settings is shown in Tables \ref{tab:mb_performance_metrics}
and \ref{tab:cd_performance_metrics}. Performance is reported for
scenarios within a 10-hour time limit (otherwise shown with a dash
``$-$''). In both tasks, \textbf{FANSITEMB} and \textbf{FANSITE-DCD}
consistently achieve best or second-best results. In the MB discovery
task, \textbf{FANSITEMB }significantly outperforms other baselines
by a large margin in challenging scenarios, including dense graphs
($\text{d100-1}$), super-large graphs ($\text{d5000}$), nonlinear
settings, and real networks, due to two advantages: (1) a universal
score estimator capturing complex relationships and varied noise distributions,
and (2) a straightforward greedy strategy for evolving MBs in poly-time.
S\textsuperscript{2}TMB, EEMB\textbf{ }and\textbf{ }TC\textbf{ }show
competitive performance on linear data thanks to solid convergence
foundations of constraint-based groups but struggle with nonlinear
data because of restricted score criteria. Moral-DAT and KMB perform
better in nonlinear settings and real networks but receive low F1
scores due to sensitive thresholds/stop conditions. In the CD task,
focus primarily lies on nonlinear settings due to their greater challenge.
SDCD\textbf{ }serves as SOTA among previous methods, surpassed by
\textbf{FANSITE-DCD}, which incorporates \textbf{FANSITEMB} to better
warm-start the subsequent DAG recovery stage from a substantially
reduced search space.
\begin{figure}[h]
\vspace{-0.4cm}

\includegraphics[width=1\columnwidth]{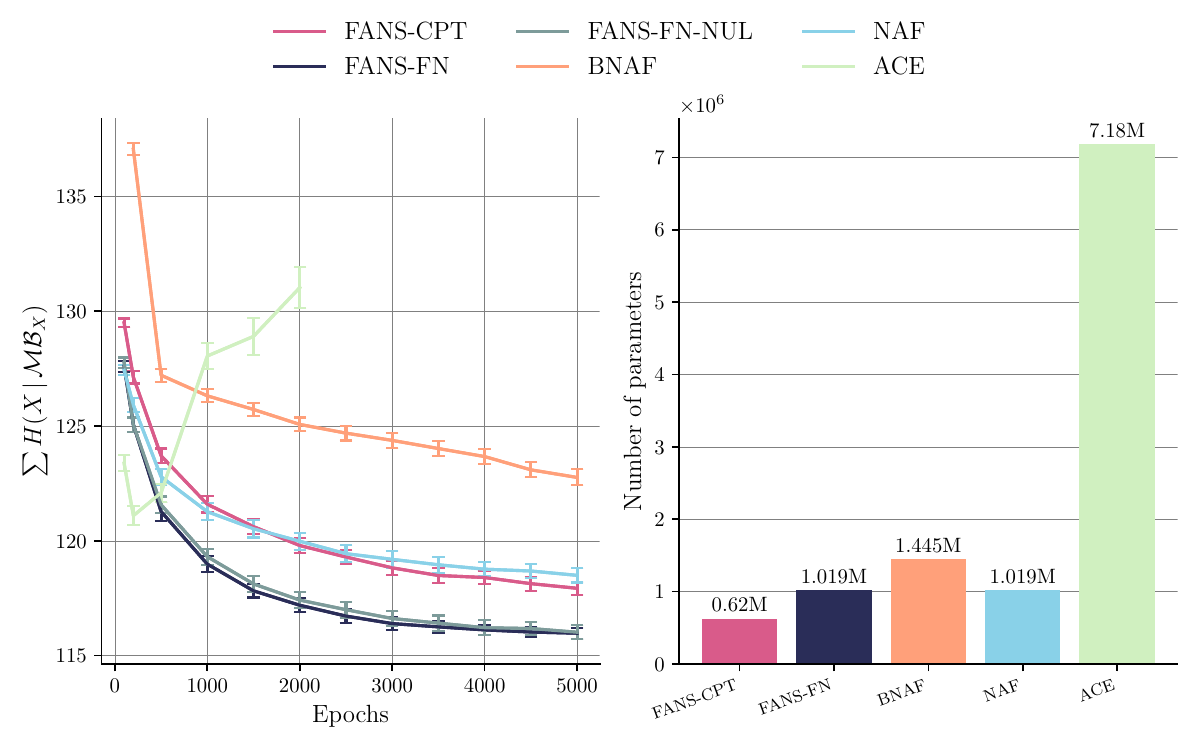}\vspace{0.6cm}\caption{Different model options for estimating our score criterion---conditional
entropy. Experiments conducted on synthetic nonlinear graphs with
$d=100$. FANS-CPT, FANS-FN, and FANS-FN-NUL represent three variants
in our proposed class of masked autoregressive flows (MAFs) tailored
for MB discovery, corresponding to: a compact version with leaf masks,
a non-compact version with leaf-masks, and a non-compact version
without leaf masks. NAF and BNAF are typical MAF classes, while ACE
\cite{strauss2021arbitrary} represents a class of energy-based estimators.
The left figure illustrates the reduction across training epochs of
$\sum H(X\mid\mathcal{MB}_{X})$, which is the total conditional entropy
as each variable conditions on its MB)---smaller values indicate
better performance. The right figure shows the number of trainable
parameters; fewer parameters imply greater efficiency. When \textbf{FANS}
is fully parameterized and trained with leaf-masking (FANS-FN), it
learns even better than the compact counterpart. However, the compact
\textbf{FANS}, with half the parameters of NAF, is expressive enough
to surpass NAF in training and estimation, demonstrating the efficacy
of our \textbf{FANS} class, as discussed in Subsection \ref{subsect:asmaf}. }\label{fig:ce_options}
\end{figure}

\vspace{-0.3cm}

\section{Conclusion}\label{sec:Conclusion}

This paper presents an efficient score-based framework for Markov
boundary (MB) discovery, employing conditional entropy from information-theoretic
learning to facilitate score objective design. A novel class of autoregressive
flows is introduced to capture complex dependencies and noise distributions,
accompanied by a theoretically guaranteed greedy parallelizable search
in poly-time, ensuring robustness, scalability, and reliability in
challenging high-dimensional and nonlinear settings. Experiments on
synthetic and real bio-genetic datasets demonstrate the method’s superior
performance in both MB and causal discovery tasks. As MB discovery
is fundamental for tasks such as feature selection and causal inference,
this advancement offers a versatile and promising solution for these
critical data mining problems.
\balance

\bibliographystyle{plain}
\bibliography{ref}

\newpage{}

\onecolumn

\begin{appendices}

\section{Proofs from the section \ref{sec:Methodology} }

\subsection{Proof of theorem \ref{thm:min_hcond}}\label{subsec:proof_a}

\textbf{Theorem }\ref{thm:min_hcond}\textbf{.} Let $T\in\mathbf{V}$
be a random variable and let $\mathcal{MB}_{T}\subset\mathbf{V}$
denote its Markov boundary. Under causal sufficiency and the faithfulness
condition, the Markov boundary $\mathcal{MB}_{T}$ of the target variable
$T$ is the unique minimal subset $\mathbf{S}\subseteq\mathbf{V}\setminus\{T\}$
that minimizes the conditional entropy $H(T\mid\mathbf{S})$. Furthermore,
for any subset $\mathbf{S}'\subseteq\mathbf{V}\setminus\{T\}$ such
that $\mathcal{MB}_{T}\subseteq\mathbf{S}'$, it holds that $H(T\mid\mathbf{S}')=H(T\mid\mathcal{MB}_{T})$.

\textit{Proof.}

By Theorem 2 in \cite{yu2021unified}, under causal sufficiency and
the faithfulness condition, the Markov boundary $\mathcal{MB}_{T}$
of $T$ is the unique minimal subset of $\mathbf{V}\setminus\{T\}$
that maximizes the mutual information $I(T;\mathbf{S})$ over all
subsets $\mathbf{S}\subseteq\mathbf{V}\setminus\{T\}$.

By the definition of mutual information,
\begin{equation}
I(T;\mathbf{S})=H(T)-H(T\mid\mathbf{S}),\label{eq:plem1}
\end{equation}
where $H(T)$ is the entropy of $T$ and $H(T\mid\mathbf{S})$ is
the conditional entropy of $T$ given $\mathbf{S}$. Since $H(T)$
is independent of the choice of $\mathbf{S}$, maximizing $I(T;\mathbf{S})$
is equivalent to minimizing $H(T\mid\mathbf{S})$. Therefore,
\begin{equation}
\mathcal{MB}_{T}=\arg\min_{\mathbf{S}\subseteq\mathbf{V}\setminus\{T\}}H(T\mid\mathbf{S}),\label{eq:plem2}
\end{equation}
and $\mathcal{MB}_{T}$ is the unique minimal subset with this property.

Now, consider any subset $\ensuremath{\mathbf{S}'\subseteq\mathbf{V}\setminus\{T\}}$
such that $\mathcal{MB}_{T}\subseteq\mathbf{S}'$. By the property
that conditioning on a larger set cannot increase conditional entropy,
we have
\begin{equation}
H(T\mid\mathbf{S}')\leq H(T\mid\mathcal{MB}_{T}).\label{eq:plem3}
\end{equation}
Since $\mathcal{MB}_{T}$ is a minimizer, equality must hold:
\begin{equation}
H(T\mid\mathbf{S}')=H(T\mid\mathcal{MB}_{T}).\label{eq:plem4}
\end{equation}
$\blacksquare$

\subsection{Proof of theorem \ref{thm:Bound}}\label{subsec:proof_b}

\textbf{Theorem }\ref{thm:Bound}\textbf{.} Assume the existence of
a universal conditional entropy approximator $\mathcal{F}_{\theta}^{*}$.
Let the forward phase of Algorithm \ref{algo:naigmb} yield $\tilde{\mathcal{MB}}_{X_{i}}^{+}$,
with $|\tilde{\mathcal{MB}}_{X_{i}}^{+}|=k'$. Denote the true Markov
boundary of $X_{i}$ as $\mathcal{MB}_{X_{i}}$ with $|\mathcal{MB}_{X_{i}}|=k$.
For any $\delta_{e}>0$,
\begin{equation}
0\leq H(X_{i}\mid\tilde{\mathcal{MB}}_{X_{i}}^{+})-H(X_{i}\mid\mathcal{MB}_{X_{i}})\leq\delta_{e}+\Delta_{N}\label{eq:ptheo1}
\end{equation}
where
\[
\Delta_{N}:=\log N\left[1-\exp\left(-\frac{k'}{k\gamma\max\{\log N,\log(1/\delta_{e})\}}\right)\right],
\]
and $\gamma$ reflects measurement noise for $N$ samples. 

For the linear Gaussian case, where the conditional entropy $H(T\mid\mathbf{S})$
has the closed form 
\begin{equation}
H(T\mid\mathbf{S})=\frac{1}{2}\left(1+\log(2\pi)+\log\frac{\det\bm{\Sigma}_{\{T\}\cup\mathbf{S}}}{\det\bm{\Sigma}_{\mathbf{S}}}\right),\label{eq:ptheo2}
\end{equation}
define:
\begin{align*}
A & :=2\pi e\sigma_{X_{i}}^{2}-\exp(2H(X_{i}\mid\mathcal{MB}_{X_{i}}))\\
A' & :=2\pi e\sigma_{X_{i}}^{2}-\exp(2H(X_{i}\mid\tilde{\mathcal{MB}}_{X_{i}}^{+}))
\end{align*}
Then:
\begin{equation}
1\leq\frac{A}{A'}\leq\frac{1}{1-\exp\left(-\lambda_{\min}(\mathbf{C},2k')\right)}\label{eq:ptheo3}
\end{equation}
where $\sigma_{X_{i}}$ is the standard deviation of $X_{i}$, and
$\lambda_{\min}(\mathbf{C},2k')$ is the minimum eigenvalue among
all $2k'\times2k'$ sub-matrices of the joint correlation matrix $\mathbf{C}$.

\textit{Proof.}

We first consider the general case. Assume the density $p(x)$ of
any $X$ is Riemann integrable. By the limiting density of discrete
points (LDDP) theorem (Theorem 8.3.1, \cite{cover1999elements}),
for sufficiently small bin width $\Delta$,
\begin{equation}
H_{\mathbb{S}}(X^{\Delta})+\log\Delta\rightarrow H(X)\quad\text{as}\quad\Delta\rightarrow0,\label{eq:ptheo4}
\end{equation}
where $H_{\mathbb{S}}(X^{\Delta})$ is the Shannon entropy of the
discretized variable. Therefore, the conditional entropy for continuous
variable $T$ and conditioning set $\mathbf{S}$ can be approximated
as
\begin{equation}
H(T\mid\mathbf{S})=H(T,\mathbf{S})-H(\mathbf{S})\approx H_{\mathbb{S}}(T,\mathbf{S})-H_{\mathbb{S}}(\mathbf{S})=H_{\mathbb{S}}(T\mid\mathbf{S}).\label{eq:ptheo5}
\end{equation}
This facilitates Theorem 2 in \cite{chen2015sequential}. Specifically,
for any $\delta_{e}>0$, the Shannon mutual information satisfies:
\begin{equation}
I_{\mathbb{S}}(X_{i};\tilde{\mathcal{MB}}_{X_{i}}^{+})\geq I_{\mathbb{S}}(X_{i};\mathcal{MB}_{X_{i}})-\delta_{e}-\Delta_{N},\label{eq:ptheo6}
\end{equation}
where
\[
\Delta_{N}:=\log N\left[1-\exp\left(-\frac{k'}{k\gamma\max\{\log N,\log(1/\delta_{e})\}}\right)\right].
\]
From Eq. \ref{eq:ptheo5} and Inequality \ref{eq:ptheo6}, we can
derive the upper bound in the general case:
\begin{equation}
0\leq H(X_{i}\mid\tilde{\mathcal{MB}}_{X_{i}}^{+})-H(X_{i}\mid\mathcal{MB}_{X_{i}})\leq\delta_{e}+\Delta_{N}.\label{eq:ptheo7}
\end{equation}

If we choose $k'\geq k\gamma\max\{\log N,\log(1/\delta_{e})\}\ln\left(\frac{\log N}{\delta_{e}}\right)$
(i.e., sufficiently many queries in the forward pass of Algorithm
\ref{algo:naigmb}), it follows directly from the analysis of Theorem
2 in \cite{chen2015sequential} that
\[
I_{\mathbb{S}}(X_{i};\tilde{\mathcal{MB}}_{X_{i}}^{+})\geq I_{\mathbb{S}}(X_{i};\mathcal{MB}_{X_{i}})-2\delta_{e},
\]
which implies that: $\ensuremath{H(X_{i}\mid\tilde{\mathcal{MB}}_{X_{i}}^{+})-H(X_{i}\mid\mathcal{MB}_{X_{i}})\leq2\delta}_{e}$. 

\textbf{Linear Gaussian Case.}

Suppose $\mathbf{x}_{\mathbf{S}}$ are standardized to $\tilde{\mathbf{x}}_{\mathbf{S}}$.
First, we prove that the entropy of a standardized Gaussian variable
set $\mathbf{S}$ is given by
\begin{equation}
\tilde{H}(\mathbf{S})=H(\mathbf{S})-\sum_{j}\log\sigma_{j},\label{eq:ptheo8}
\end{equation}
where $\sigma_{j}$ is the standard deviation of $X_{j}\in\mathbf{S}$. 

Denote $s:=|\mathbf{S}|$. Let $\mathbf{x}_{\mathbf{S}}\sim\mathcal{N}(\boldsymbol{\mu}_{\mathbf{S}},\boldsymbol{\Sigma}_{\mathbf{S}})$,
where $\boldsymbol{\mu}_{\mathbf{S}}\in\mathbb{R}^{s}$ is the mean
vector and $\boldsymbol{\Sigma}_{\mathbf{S}}\in\mathbb{R}^{s\times s}$
is the covariance matrix. The Gaussian entropy $H(\mathbf{S})$ is
given by:
\begin{equation}
H(\mathbf{S})=\frac{1}{2}\log\left((2\pi e)^{s}\det(\boldsymbol{\Sigma}_{\mathbf{S}})\right).\label{eq:ptheo9}
\end{equation}

Let $\mathbf{S}:=\{X_{i_{1}},X_{i_{2}},\ldots,X_{i_{s}}\}$ and let
$\sigma_{j}$ denote the standard deviation of $X_{j}$ for all $j\in\{i_{1},i_{2},\ldots,i_{s}\}$.
Using the factorization $\boldsymbol{\Sigma}_{\mathbf{S}}=\mathbf{D}_{\mathbf{S}}\mathbf{C_{S}}\mathbf{D_{S}}$,
where $\mathbf{D_{S}}=\text{diag}(\sigma_{i_{1}},\ldots,\sigma_{i_{s}})$,
and $\mathbf{C}_{\mathbf{S}}$ is sub-correlation matrix of $\mathbf{C}$
corresponding to variables in $\mathbf{S}$ from $\{X_{1},X_{2},\ldots,X_{d}\}$,
we get:
\begin{equation}
\det(\boldsymbol{\Sigma}_{\mathbf{S}})=\det(\mathbf{D}_{\mathbf{S}})^{2}\det(\mathbf{C}_{\mathbf{S}})=\left(\prod_{j=1}^{s}\sigma_{i_{j}}\right)^{2}\det(\mathbf{C}_{\mathbf{S}}).\label{eq:ptheo10}
\end{equation}

So the entropy becomes:
\begin{equation}
H(\mathbf{S})=\frac{1}{2}\log\left((2\pi e)^{s}\left(\prod_{j=1}^{s}\sigma_{i_{j}}\right)^{2}\det(\mathbf{C}_{\mathbf{S}})\right).\label{eq:ptheo11}
\end{equation}
Now define the standardization:
\begin{equation}
\tilde{\mathbf{x}}_{\mathbf{S}}=\mathbf{D_{S}}^{-1}(\mathbf{x}_{\mathbf{S}}-\boldsymbol{\mu}_{\mathbf{S}})\sim\mathcal{N}(\mathbf{0},\mathbf{C}_{\mathbf{S}}).\label{eq:ptheo12}
\end{equation}

The entropy $\tilde{H}(\mathbf{S})$ corresponding to $\tilde{\mathbf{x}}_{\mathbf{S}}$
is:
\begin{equation}
\tilde{H}(\mathbf{S})=\frac{1}{2}\log\left((2\pi e)^{s}\det(\mathbf{C}_{\mathbf{S}})\right).\label{eq:ptheo13}
\end{equation}
Now subtract:
\begin{align}
H(\mathbf{S})-\tilde{H}(\mathbf{S}) & =\frac{1}{2}\log\left((2\pi e)^{s}\left(\prod_{j=1}^{s}\sigma_{i_{j}}\right)^{2}\det(\mathbf{C}_{\mathbf{S}})\right)-\frac{1}{2}\log\left((2\pi e)^{s}\det(\mathbf{C}_{\mathbf{S}})\right)\label{eq:ptheo14}\\
 & =\frac{1}{2}\log\left(\left(\prod_{j=1}^{s}\sigma_{i_{j}}\right)^{2}\right)=\log\left(\prod_{j=1}^{s}\sigma_{i_{j}}\right)=\sum_{j=1}^{s}\log\sigma_{i_{j}}.\nonumber 
\end{align}

Next, let $\mathbf{C}{}_{\mathbf{S}}$ and $\mathbf{C}_{\{T\}\cup\mathbf{S}}$
denote the sub-matrices of the correlation matrix $\mathbf{C}$ corresponding
to the variable sets $\mathbf{S}$ and $\{T\}\cup\mathbf{S}$, respectively
(i.e., the covariance matrix after standardization). From Lemma 3.6
in \cite{das2011submodular}, we have:
\begin{equation}
\mathbf{C}_{\{T\}\cup\mathbf{S}}=\begin{pmatrix}1 & \mathbf{b}^{\top}\\
\mathbf{b} & \mathbf{C_{S}}
\end{pmatrix},\label{eq:ptheo15}
\end{equation}
where $\mathbf{b}$ is the covariance vector between $T$ and variables
in $\mathbf{S}$.

The conditional entropy is then
\begin{align}
H(T\mid\mathbf{S}) & =H(T,\mathbf{S})-H(\mathbf{S})\nonumber \\
 & =\tilde{H}(T,\mathbf{S})+\log\sigma_{T}+\sum_{j=1}^{s}\log\sigma_{i_{j}}-\tilde{H}(\mathbf{S})-\sum_{j=1}^{s}\log\sigma_{i_{j}}\nonumber \\
 & =\tilde{H}(T,\mathbf{S})+\log\sigma_{T}-\tilde{H}(\mathbf{S})\nonumber \\
 & =\tilde{H}(T\mid\mathbf{S})+\log\sigma_{T}\nonumber \\
 & =\frac{1}{2}\left(1+\log(2\pi)+\log\frac{\det\mathbf{C}_{\{T\}\cup\mathbf{S}}}{\det\bm{\mathbf{C}}_{\mathbf{S}}}\right)+\log\sigma_{T}\label{eq:ptheo16}\\
 & =\frac{1}{2}\left(1+\log(2\pi)+\log(1-\mathbf{b^{\mathsf{T}}}\mathbf{C}_{\mathbf{S}}^{-1}\mathbf{b})\right)+\log\sigma_{T}\nonumber 
\end{align}
It is easy to see that the forward pass procedure in our Algorithm
\ref{algo:naigmb} aligns with the forward regression (forward selection)
process defined in Definition 3.1 of \cite{das2011submodular}. From
Eq. \ref{eq:ptheo16}, it can be deduced that minimizing our conditional
entropy score $H(X_{i}\mid\tilde{\mathcal{MB}}_{X_{i}}^{+})$ is equivalent
to maximizing the term
\[
R_{X_{i},\tilde{\mathcal{MB}}_{X_{i}}^{+}}^{2}:=\mathbf{b}^{\mathsf{T}}\left(\mathbf{C}_{\tilde{\mathcal{MB}}_{X_{i}}^{+}}\right)^{-1}\mathbf{b},
\]
where $\mathbf{b}$ is the covariance vector between $X_{i}$ and
variables in $\tilde{\mathcal{MB}}_{X_{i}}^{+}$. Hence, from Theorem
3.2 in \cite{das2011submodular}, it is straightforward to obtain:
\begin{equation}
R_{X_{i},\tilde{\mathcal{MB}}_{X_{i}}^{+}}^{2}\geq\left(1-\exp\left(-\lambda_{\min}(\mathbf{C},2k')\right)\right)R_{X_{i},\mathcal{MB}_{X_{i}}}^{2},\label{eq:ptheo17}
\end{equation}
where $\ensuremath{R_{X_{i},\mathcal{MB}_{X_{i}}}^{2}}$ corresponds
to the optimal case. $R_{T,\mathbf{S}}^{2}$ is weakly sub-modular
\cite{das2011submodular}.

Let:
\[
A:=2\pi e\sigma_{X_{i}}^{2}-\exp(2H(X_{i}\mid\mathcal{MB}_{X_{i}})),\quad A':=2\pi e\sigma_{X_{i}}^{2}-\exp(2H(X_{i}\mid\tilde{\mathcal{MB}}_{X_{i}}^{+})).
\]
From Eq. \ref{eq:ptheo16} and Inequality \ref{eq:ptheo17}, it follows
that
\begin{equation}
\frac{A}{A'}=\frac{R_{X_{i},\mathcal{MB}_{X_{i}}}^{2}}{R_{X_{i},\tilde{\mathcal{MB}}_{X_{i}}^{+}}^{2}}\leq\frac{1}{1-\exp(-\lambda_{\min}(\mathbf{C},2k'))},\label{eq:ptheo18}
\end{equation}
where $\mathbf{C}$ is the joint correlation matrix, i.e., the joint
covariance matrix $\mathbf{\Sigma}$ of all variables after standardization.

The lower bound of $1$ in Inequality \ref{eq:ptheo3} follows since
$\mathcal{MB}_{X_{i}}$ is the true minimizer of $H(X_{i}\mid\mathbf{S})$.
$\blacksquare$

\subsection{Proof of theorem \ref{thm:soundness-cemb}}\label{subsec:proof_c}

\textbf{Theorem }\ref{thm:soundness-cemb}\textbf{.} Assume infinite
observational data, faithfulness, causal sufficiency, and that $\mathcal{F}_{\theta}^{*}$
is a universal conditional entropy approximator. Then Algorithm \ref{algo:naigmb}
identifies exactly all variables in the Markov boundary of a target
node.

\textit{Proof.}

By our Theorem\textbf{ }\ref{thm:Bound}, under the stated conditions,
the growing phase of Algorithm \ref{algo:naigmb} guarantees that
$\tilde{\mathcal{MB}}_{X_{i}}^{+}$ contains the true Markov boundary
if allowed to grow sufficiently large, i.e., $\mathcal{MB}_{X_{i}}\subseteq\tilde{\mathcal{MB}}_{X_{i}}^{+}$.
Notably, for sparse Bayesian networks, the selection process may terminate
early, yielding a relatively small $\tilde{\mathcal{MB}}_{X_{i}}^{+}$
compared to the total number of variables.

We next demonstrate that the shrinking phase eliminates all extraneous
variables, i.e., each $X_{j}\in\tilde{\mathcal{MB}}_{X_{i}}^{+}\setminus\mathcal{MB}_{X_{i}}$,
via induction on their count. At each iteration, a variable not in
$\mathcal{MB}_{X_{i}}$ is removed from $\tilde{\mathcal{MB}}_{X_{i}}^{+}$,
and this process continues until convergence. This follows directly,
as removing any non-Markov boundary variable does not increase $H(X_{i}\mid\tilde{\mathcal{MB}}_{X_{i}}^{+})$,
which already attains the minimum conditional entropy achievable (according
to our Theorem \ref{thm:min_hcond}). $\blacksquare$
\label{sec:app_a}

\section{Reproduction details}

\textbf{\textit{Hyperparameter choice.}} We perform a grid-search
to adopt best configurations. \textit{Hyperparameters of }\textbf{\textit{FANS}}\textit{:}
Transformer is a deep sigmoidal flow \cite{huang2018neural}, with
1 layer and $w,a,b$ as $4$-dimensional vectors. Models are trained
for up to 5000 epochs. Batch size is 64 for graphs with $d<100$ and
256 for $d\geq100$. In conditioner, output layer has 20 blocks for
$d<100$ and 16 for $d\geq100$. One hidden layer with $\alpha_{B}=6$
blocks is used for all $d$. For $d\geq100$, compact version of \textbf{FANS}
with $2*M$ nodes per hidden block is used; non-compact version is
used for $d<100$. Maximum size of observed subsets: $M=d-1$ for
$d\leq20$, $M=20$ for $30\leq d<100$, and $M=30$ for $d\geq100$.
\textit{Hyperparameters of Algorithm} \ref{algo:naigmb}: $K=1000$,
$\epsilon_{g}=0.005$, $\epsilon_{s}=0.002$ for sparse graphs (node
degree is less than 2); $\epsilon_{g}=\epsilon_{s}=0.001$ for dense
graphs, $\rho=50$ for $d=5000$ and $\rho=15$ otherwise. Hyperparameters
for DAG estimation in \textbf{FANSITE-DCD} set to defaults of SDCD
\cite{nazaret2024stable}. Real/semi-real datasets are standardized
beforehand.\textbf{\textit{ }}

\textbf{\textit{Technical details.}} All experiments were conducted
on Apple M3 CPU and A100 GPU.
\label{sec:app_b}

\end{appendices}
\end{document}